\PassOptionsToPackage{capitalize}{cleveref}
\PassOptionsToPackage{headheight=28pt}{geometry}
\documentclass{ailab}
\usepackage[utf8]{inputenc}
\usepackage[T1]{fontenc}
\usepackage{amsthm}
\newtheorem{definition}{Definition}
\newtheorem{proposition}{Proposition}
\usepackage{booktabs}
\usepackage{array}
\usepackage{tabularx}
\usepackage{colortbl}
\newcolumntype{Y}{>{\raggedright\arraybackslash}X}
\usepackage{amsmath,amssymb,amsfonts}
\usepackage{microtype}
\usepackage{xcolor}
\usepackage{graphicx}
\usepackage{algorithm}
\usepackage{algpseudocode}
\algrenewcommand{\algorithmicrequire}{\textnormal{\bfseries Require:}}
\algrenewcommand{\algorithmicensure}{\textnormal{\bfseries Ensure:}}
\algrenewcommand{\algorithmiccomment}[1]{\hfill{\footnotesize\color{gray}\itshape\# #1}}
\usepackage{float}
\usepackage{enumitem}
\usepackage{subcaption}
\usepackage[most]{tcolorbox}
\usepackage{wrapfig}
\usepackage{listings}
\usepackage{titletoc} %
\usepackage{hyperref}
\hypersetup{
    colorlinks=true,
    linkcolor=metabg,
    citecolor=metabg,
    urlcolor=metabg,
    filecolor=metabg
}
\usepackage{url}
\usepackage{needspace}

\titlecontents{appsection}
  [15pt]
  {\addvspace{8pt}\normalfont\bfseries}
  {\contentslabel[\thecontentslabel\hfill]{15pt}}
  {\hspace*{-15pt}}
  {\hfill\contentspage}

\titlecontents{appsubsection}
  [39pt]
  {\addvspace{2pt}\normalfont}
  {\contentslabel[\thecontentslabel\hfill]{24pt}}
  {\hspace*{-24pt}}
  {\titlerule*[8pt]{.}\contentspage}

\definecolor{lightblue}{RGB}{173,216,230}
\definecolor{papernavy}{HTML}{29465B}
\definecolor{paperline}{HTML}{D5DEE6}
\definecolor{paperbg}{HTML}{F6F8FA}
\lstdefinestyle{promptlisting}{
    basicstyle=\fontsize{9}{12}\selectfont\ttfamily,
    aboveskip=2pt,
    belowskip=2pt,
    breaklines=true,
    breakatwhitespace=true,
    columns=fullflexible,
    keepspaces=true,
    showstringspaces=false
}
\tcbset{papercard/.style={
    enhanced,
    colback=white, colframe=paperline, boxrule=0pt,
    frame hidden,
    borderline west={1.2pt}{0pt}{papernavy},
    borderline south={0.35pt}{0pt}{paperline},
    colbacktitle=paperbg, coltitle=papernavy,
    fonttitle=\sffamily\bfseries\fontsize{9.5}{12}\selectfont,
    titlerule=0pt, toptitle=6pt, bottomtitle=6pt,
    arc=0pt, outer arc=0pt,
    left=14pt, right=14pt, top=10pt, bottom=10pt,
    before skip=10pt, after skip=12pt
}}
\newtcolorbox{synthprompt}[1][]{
    papercard,
    breakable,
    colback=paperbg, colbacktitle=papernavy!12!white, coltitle=papernavy,
    toptitle=5pt, bottomtitle=5pt,
    title after break={Prompt template (continued)},
    #1
}
\newtcolorbox{wrapprompt}[1][]{
    papercard,
    breakable,
    colback=paperbg, colbacktitle=papernavy!12!white, coltitle=papernavy,
    toptitle=5pt, bottomtitle=5pt,
    #1
}
\newtcolorbox{walkcard}[2][]{
    papercard,
    breakable,
    title={#2},
    title after break={#2\space (continued)},
    fonttitle=\sffamily\bfseries\fontsize{9}{11}\selectfont,
    fontupper=\fontsize{9}{12}\selectfont,
    toptitle=2pt, bottomtitle=2pt,
    top=3pt, bottom=3pt, before skip=3pt, after skip=4pt,
    before upper={\setlength{\parskip}{2pt}%
      \setlength{\abovedisplayskip}{3pt}\setlength{\belowdisplayskip}{3pt}},
    #1
}
\newcommand{\walkstep}[2]{%
    \par\Needspace{5\baselineskip}\addvspace{6pt}\noindent
    \tcbox[on line,colback=papernavy,colframe=papernavy,boxrule=0pt,
      arc=1pt,boxsep=0pt,left=5pt,right=5pt,top=3pt,bottom=3pt]{%
      \color{white}\sffamily\bfseries\fontsize{8}{10}\selectfont STEP #1}%
    \hspace{8pt}{\sffamily\bfseries\fontsize{9}{11}\selectfont\color{papernavy}#2}\par\nobreak\vspace{2pt}}
\newcommand{\walkdivider}{%
    \par\addvspace{2pt}{\color{paperline}\hrule height 0.35pt}%
    \par\addvspace{3pt}}
\newcommand{\walkrecord}[2]{\item[$j=#1$]\parbox[t]{\linewidth}{#2}}
\theoremstyle{remark}
\newtheorem{remark}{Remark}
\title{More Memory, Worse Agents:
Error Reproduction and Anti-Persistence in LLM Agents}
\author[1]{Ruhan Wang}
\author[2]{Kishan Panaganti}
\author[1]{Dongruo Zhou}
\affiliation[1]{Indiana University}
\affiliation[2]{Tencent HY}
\email{ruhwang@iu.edu}
\email{kpb@global.tencent.com}
\email{dz13@iu.edu}
\date{October 19, 2026}
\hypersetup{
    pdftitle={More Memory, Worse Agents: Error Reproduction and Anti-Persistence in LLM Agents},
    pdfauthor={Ruhan Wang, Kishan Panaganti, Dongruo Zhou}
}

\abstract{
LLM agents can improve across tasks without updating model parameters by storing experience from their own executions in external memory and reusing it to guide later decisions. A central design question is whether to retain only observed trajectories and outcomes or also the rules, summaries, and skills inferred from them. Retaining these abstractions allows agents to reuse strategies without repeatedly deriving them from raw trajectories. However, errors in stored abstractions can propagate to subsequent decisions through retrieval and reuse. They can also shape execution trajectories, which then become the basis for new abstractions that reproduce the original errors. We term this feedback mechanism \emph{error reproduction}: errors are regenerated through a cycle of abstraction, execution, and re-abstraction, extending their persistence beyond direct reuse of the original faulty knowledge. We propose \textsc{Anti-Persistence}, a memory design that accumulates trajectories and outcomes while reconstructing task-specific guidance on demand. An adaptive selector uses prior outcomes to decide whether and how to synthesize guidance, which is discarded after execution. This prevents its direct reuse as a stored object, although indirect feedback through retained trajectories remains possible. Across different benchmarks and backbones, \textsc{Anti-Persistence} achieves higher overall held-out success and the lowest average LLM-call cost among the evaluated memory designs. These results support separating experience accumulation from abstraction persistence in memory-based self-improvement.

}

\begin{document}
\thispagestyle{firstheader}
\maketitle

\section{Introduction}
\label{sec:intro}

Large language model (LLM) agents solve tasks through multi-step interactions with an environment, producing trajectories of actions and observations~\citep{yao2022react}. We study \emph{memory-based self-improvement}, in which an agent stores experience and outcomes from its own executions in external memory, retrieves that memory to guide later tasks, and updates it with the resulting experience~\citep{shinn2023reflexion,zhao2024expel,song2024trial,wang2023voyager}. Together, repeated execution, memory updates, and reuse form a feedback loop across tasks~\citep{xiong2025memory,du2026memory}: with the LLM parameters held fixed, the agent adapts through changes in the experience available for decision-making and in how that experience is used. Although this process is intended to improve subsequent performance, its effectiveness must be established empirically, because the content retained by the agent can affect both the next task and later rounds of adaptation.

Stored experience can take the form of observed interaction records, which preserve trajectories and task outcomes, or inferred abstractions, which express rules, skills, summaries, or heuristics derived from those records. Because long interaction histories contain both reusable procedures and episode-specific details, distilling them into compact guidance can make relevant experience easier to retrieve and use. Many agent systems therefore extract abstractions from past executions and retain them for subsequent tasks. This compression nevertheless requires judgments about which details matter and under what conditions a pattern generalizes; a rule that succeeds in one task may depend on conditions absent from another. The value of abstraction thus depends not only on identifying a useful pattern but also on reassessing its applicability when it is reused.

The effects of a mistaken abstraction can extend beyond the task in which it is retrieved and used. A stored rule can influence an agent's actions, and the resulting execution records may become inputs for synthesizing new abstractions in subsequent self-improvement rounds (\Cref{fig:overview}). Even when the original rule is not supplied to the synthesis step, these records may still contain erroneous actions taken under its guidance. When extracting lessons from these records, the agent may infer a new rule that reproduces the same error, allowing the error to persist indirectly through execution history. We call this possible feedback mechanism \emph{error reproduction}, distinguishing the reappearance of an error in a newly generated abstraction from its propagation through retrieval and reuse. Both trajectories and abstractions can carry influence across tasks, with the extent of that influence depending on how the stored content is retrieved and used.

Memory maintenance can improve stored content through scoring, pruning, or refinement~\citep{qiu2026autorefine}, but deleting an abstraction does not by itself revise trajectories generated under its influence or remove abstractions subsequently derived from those trajectories. Whether reproduction continues depends jointly on which records are retrieved, what is synthesized from them, and which outputs are admitted to memory. These dependencies motivate examining the need for inferred content to persist alongside the procedures used to maintain it, leading to the following question:

{\centering
How can LLM agents benefit from abstraction in memory-based self-improvement while limiting error reproduction?
\par}

We propose \textsc{Anti-Persistence}, a memory module that retains observed interaction records while synthesizing abstractions for the current task. An adaptive selector uses prior outcomes to determine whether and how to transform retrieved experience into guidance, which supports the current execution and is discarded afterward. Subsequent tasks can revisit the accumulated records to construct guidance suited to their own context, allowing the agent to adapt through experience accumulation and mode selection while the synthesis procedure remains fixed and its outputs remain temporary. Our contributions are threefold.
\begin{itemize}[leftmargin=*, nosep]

\item We introduce a memory design for self-improvement that separates the accumulation of execution experience from the temporary use of abstractions. Raw trajectories and task outcomes are retained, while inferred rules, summaries, and heuristics are used only for the current task and discarded afterward.

\item We develop an adaptive mode selector for on-demand abstraction that chooses between using retrieved records directly and synthesizing behavioral patterns, task strategies, or corrective guidance for each task. This enables a fixed synthesis procedure to support successive tasks without persisting its outputs.

\item We evaluate multiple rounds of self-improvement across three benchmarks and three backbones. \textsc{Anti-Persistence} achieves a higher overall success rate on held-out tasks at a lower average LLM-call cost. We also find associations between earlier and subsequently generated abstractions, providing empirical evidence consistent with error reproduction.

\end{itemize}

\begin{figure}[t]
\centering
\includegraphics[width=\linewidth]{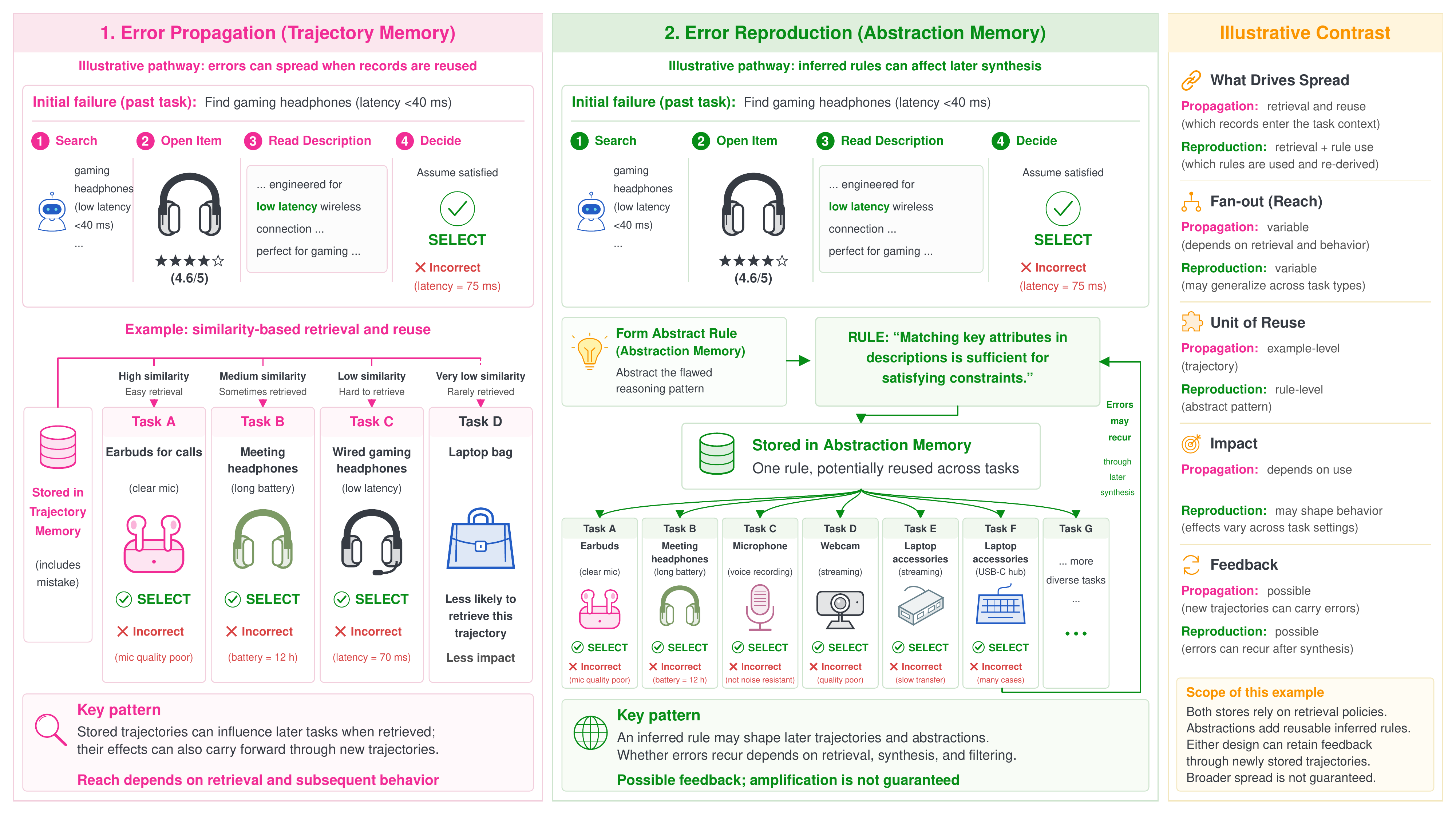}
\vspace{-0.2in}
\caption{
 \small Error propagation and reproduction in agent memory. 
\textbf{Left:} An illustrative error in a stored trajectory is reused through retrieval. 
\textbf{Center:} A faulty inferred rule may shape later trajectories and be encoded again through subsequent synthesis. 
\textbf{Right:} Possible routes of influence. Both designs can retain feedback through stored trajectories; the examples imply neither a fixed ordering of their reach nor inevitable amplification.
}
\label{fig:overview}
\vspace{10pt}
\end{figure}

\section{Related Work}
\textbf{Memory-based self-improvement in LLM agents.}
External memory allows agents to incorporate their own interaction experience into later decisions~\citep{yao2022react,song2024trial}, with existing approaches differing in the content retained across updates. Trajectory-based methods store raw or lightly processed episodes for similarity-based retrieval~\citep{xiong2025memory,du2026memory}, whereas other methods retain verbal reflections, cross-task rules, summaries, or executable skills~\citep{shinn2023reflexion,zhao2024expel,wang2023voyager,chen2024automanual,yang2026autoskill,zhang2026memskill}. Related forms of abstraction also appear in production agent systems~\citep{zhang2026equipping,yang2024swe,wang2024openhands}. We examine how the choice of persistent content affects later memory updates through the trajectories it helps generate.

\textbf{Error propagation in iterative systems.}
Prior work examines how generated content can bias later behavior through repeated use~\citep{xiong2025memory,wyllie2024fairness,briesch2023large}, while related model-collapse studies show that repeated training on self-generated data can reduce distributional fidelity~\citep{shumailov2024ai,seddik2024bad,gerstgrasser2024model,feng2024beyond}. Although parameter training differs from the external-memory updates studied here, these findings motivate examining feedback from generated experience. We focus on \emph{error reproduction}, in which abstractions shape trajectories that inform subsequent synthesis: propagation concerns how an existing error affects behavior, whereas reproduction concerns whether that influence is encoded in a new abstraction. Because deleting the original abstraction does not by itself remove stored descendants, this dependence on earlier outputs creates a possible feedback loop across iterations~\citep{perdomo2020performative,pagan2023classification}, even though maintenance can also target the descendant entries.

\textbf{Memory maintenance and reliability.}
Related approaches improve memory or agent reliability through pruning, scoring, refinement, and learned operations~\citep{qiu2026autorefine,packer2023memgpt,park2023generative,madaan2023self,schick2023toolformer,zhong2024memorybank}. In systems that retain abstractions, these maintenance procedures determine which inferred content remains available for direct retrieval across tasks. \textsc{Anti-Persistence} addresses the complementary question of persistence by synthesizing task-specific abstractions on demand and discarding them after use, while retaining the resulting trajectories for later retrieval and synthesis. This design removes direct abstraction reuse but still permits indirect feedback through experience.

\section{Anti-Persistence: A Memory Design for Self-Improving Agents}
\label{sec:method}
\textsc{Anti-Persistence} constructs temporary guidance from accumulated interaction records, combining on-demand synthesis with adaptive mode selection. We first formalize the memory update process and then describe how these components support task execution.
\subsection{Preliminaries}
\label{subsec:problem-setup}

We consider an agent with fixed backbone parameters $\theta$ that processes a sequence of tasks, receiving task $\tau_k$ and producing trajectory $e_k$ and outcome $y_k$ at training-task update $k$. Each experimental iteration $\ell$ comprises multiple such updates followed by held-out evaluation with memory and related state frozen.

\Needspace{6\baselineskip}
\textbf{Memory.}
The agent maintains a trajectory memory
$
M_k=\{(\tau_j,e_j,y_j):j<k\},
$
which stores previous executions and their outcomes. An operator $\phi$ generates an abstraction from the available records:
$
s=\phi(H),\quad H\subseteq M,
$
where $M$ is the trajectory store available at synthesis time, with task, mode, and decoding arguments omitted for clarity. The abstraction store $S_k$ contains the summaries, rules, strategies, or skills currently retained; because maintenance may delete or transform entries, it need not include every abstraction generated previously.

\textbf{Retrieval and generation.}
Given a new task $\tau_k$, the agent retrieves a context
$
C_k=R(\tau_k,M_k\cup S_k,B),
$
where $R$ is a retrieval operator and $B$ is the retrieval budget. The agent then generates a trajectory according to
$
e_k\sim \pi_\theta(\cdot\mid \tau_k,C_k),
$
where $\pi_\theta$ denotes the underlying policy, such as an LLM-based agent.

\textbf{Memory update.}
After each task, the agent stores the task, execution trajectory, and outcome:
$
M_{k+1}=M_k\cup\{(\tau_k,e_k,y_k)\}.
$
Methods that retain abstractions also update a separate store $S_k$. Let $\mathcal A_k$ denote new candidates synthesized from available records. Then
$
S_{k+1}=\mathcal U_k(S_k,\mathcal A_k;M_{k+1}).
$
Here, $\mathcal U_k$ decides how to incorporate these candidates and maintain existing entries, with $M_{k+1}$ available as evidence. It may simply append all candidates or apply filtering, pruning, or merging.

Anti-persistence updates $M_k$ but keeps $S_k=\emptyset$: guidance generated for the current task is discarded after execution. The execution records remain available for later retrieval and synthesis.

\textbf{Feedback through experience.}
A stored abstraction $s_{\mathrm{parent}}$ can enter the agent's context during task execution, which produces trajectory $e_j$. Once execution is complete, its record may be included in retrieved records $H$ and used to synthesize a new abstraction $s_{\mathrm{candidate}}$:
$
s_{\mathrm{parent}}\longrightarrow e_j\longrightarrow H\longrightarrow s_{\mathrm{candidate}}.
$
An error in the old abstraction may thus influence execution and reappear in newly synthesized guidance. Appendix~\ref{app:dynamics} separately analyzes exposure to flagged abstractions, generation of flagged candidates, and their admission to persistent memory.
\begin{algorithm}[t]
\caption{\textsc{Anti-Persistence}: local mode selection and task execution}
\label{alg:synth}
\fontsize{9}{11.5}\selectfont
\begin{algorithmic}[1]
\setlength{\itemsep}{1pt}
\Require Task $\tau$; stores $M_k,\mathcal L_k$; modes $\mathcal Z$; budgets $B,\hat B$; coefficients $c,\gamma,\lambda$.
\Statex \hspace{\algorithmicindent}\textit{Local branch:} nonempty trusted neighborhood; positive weight for every mode;
\Statex \hspace{\algorithmicindent}no warmup or exploration override.
\Ensure Trajectory $e_\tau^{(k)}$, outcome $y_\tau^{(k)}$, and updated stores $M_{k+1},\mathcal L_{k+1}$.
\State $\mathcal{N}_{\hat B}(\tau)\gets R(\tau,\mathcal L_k,\hat B)$ \Comment{Retrieve local feedback}
\State Index entries $(\tau_i,\mu_i,y_i,r_i,t_i)$ by $\mathcal I_{\hat B}(\tau)=\{1,\ldots,|\mathcal N_{\hat B}(\tau)|\}$
\State $t_{\max}\gets\max_{i\in\mathcal I_{\hat B}(\tau)}t_i$
\For{$i\in\mathcal I_{\hat B}(\tau)$}
    \State $w_i\gets\max(\mathrm{sim}(\tau_i,\tau),0)\cdot\gamma^{t_{\max}-t_i}$ \Comment{Similarity and recency}
\EndFor
\For{each mode $\mu\in\mathcal Z$}
    \State $I_\mu\gets\{i\in\mathcal I_{\hat B}(\tau):\mu_i=\mu\}$;\quad $n_\mu\gets|I_\mu|$
    \State $\hat r(\mu\mid\tau)\gets\frac{\sum_{i\in I_\mu}w_i r_i}{\sum_{i\in I_\mu}w_i}$ \Comment{Weighted reward}
    \State $U(\mu\mid\tau)\gets\hat r(\mu\mid\tau)+c\sqrt{\frac{\log|\mathcal N_{\hat B}(\tau)|}{\max(n_\mu,1)}}$ \Comment{Add exploration bonus}
\EndFor
\State $\mu^\star\gets\arg\max_{\mu\in\mathcal Z}U(\mu\mid\tau)$
\State $H_\tau\gets R_{\mu^\star}(\tau,M_k,B)$ \Comment{Mode-specific retrieval}
\State $\sigma_\tau\gets\phi(H_\tau,\mu^\star)$ \Comment{$\phi(H_\tau,\mathrm{none})=H_\tau$}
\State $(e_\tau^{(k)},y_\tau^{(k)})\gets\pi_\theta(\tau;\,C_\tau^{(k)}=\sigma_\tau)$ \Comment{Execute with guidance}
\State $\kappa_\tau\gets\text{realized extra synthesis-call count}$
\State $r_\tau^{(k)}\gets\operatorname{clip}(y_\tau^{(k)}-\lambda\kappa_\tau,-1,1)$ \Comment{Account for synthesis cost}
\State $M_{k+1}\gets M_k\cup\{(\tau,e_\tau^{(k)},y_\tau^{(k)})\}$;\quad $\mathcal L_{k+1}\gets\mathcal L_k$
\If{the attempt is eligible (\cref{app:method:outcome-log})}
    \State Append $(\tau,\mu^\star,y_\tau^{(k)},r_\tau^{(k)},t_{\mathrm{new}})$ to $\mathcal L_{k+1}$
\EndIf
\State \textnormal{\bfseries Discard} $\sigma_\tau$ \Comment{Keep synthesized guidance temporary}
\end{algorithmic}
\end{algorithm}
\subsection{Algorithm Design}
Anti-persistence separates the use of an abstraction from its storage for later reuse (\cref{subsec:problem-setup}). It retrieves interaction records, generates guidance when needed, and discards that guidance after execution.

\begin{tcolorbox}[
    breakable,
    colback=gray!4!white,
    colframe=gray!55!black,
    boxrule=0.4pt,
    left=4pt,
    right=4pt,
    top=1pt,
    bottom=1pt,
    before skip=10pt,
    after skip=10pt
]
\centering\scriptsize\ttfamily
\textbf{Persistent abstraction:}\quad
trajectory $\to$ \colorbox{red!15}{abstraction} $\to$ storage $\to$ reuse \\[2pt]
\textbf{Anti-persistence (ours):}\quad
trajectory $\to$ storage $\to$ retrieval $\to$ \colorbox{green!18}{on-demand abstraction}
\end{tcolorbox}

\textbf{What persists.}
The agent retains two kinds of records. Trajectory memory $M_k$ stores tasks, executions, and outcomes that can be retrieved as examples. The outcome log $\mathcal L_k$, defined below, records mode choices and their feedback to support later choices. Neither store retains synthesized guidance, and the abstraction store $S_k$ remains empty.

These records preserve what occurred, without certifying that the agent's actions or explanations were correct. Temporary guidance may affect a stored trajectory and thereby influence later synthesis. Discarding guidance therefore prevents its direct reuse as a stored object but leaves indirect feedback through trajectories possible.

\textit{How to use retrieved experience.}
For each task, the agent chooses whether to use interaction records directly or synthesize guidance from them. We define four modes:
\begin{align}
\mathcal{Z}
=
\{
\mathrm{none},\;
\mathrm{imitation},\;
\mathrm{strategic},\;
\mathrm{corrective}
\}.
\end{align}
The $\mathrm{none}$ mode supplies retrieved records directly, so it uses memory without an additional synthesis call. The $\mathrm{imitation}$ mode extracts reusable action patterns and their conditions from successful trajectories. The $\mathrm{strategic}$ mode uses successful trajectories to organize higher-level subgoals and plans. The $\mathrm{corrective}$ mode analyzes failed trajectories to generate avoidance guidance, using a successful trajectory as a contrast when available. Full definitions appear in \cref{app:method:modes}.

\textit{Feedback for mode selection.}
The agent uses past outcomes to assess which mode may suit the current task. We formulate this choice as a lightweight contextual bandit over $\mathcal Z$, supported by the outcome log
\begin{align}
\mathcal{L}_k
=
\{
(\tau_j,\mu_j,y_j,r_j,t_j):j<k,\ j\in\mathcal J_k
\},
\end{align}
where $\mathcal J_k$ indexes eligible completed attempts (\cref{app:method:outcome-log}). Each entry records the task, selected mode, raw outcome $y_j$, cost-adjusted reward $r_j$, and logical timestamp $t_j$. The reward guides mode selection, while the raw outcome preserves task performance separately from the cost penalty.

\textit{Two retrieval steps.}
On the local-selection branch, the controller first retrieves similar tasks from $\mathcal L_k$ to choose a mode. It then retrieves trajectories from $M_k$ according to that mode. The first retrieval supplies feedback for choosing how to use experience; the second supplies the experience itself. Their budgets are $\hat B$ and the mode-specific base budget $B$, respectively.

\textbf{Selecting a mode from local feedback.}
For task $\tau$, the first retrieval returns up to $\hat B$ nearest outcome-log entries,
$
\mathcal{N}_{\hat{B}}(\tau)
=
\{(\tau_i,\mu_i,y_i,r_i,t_i)\}_{i\in \mathcal{I}_{\hat{B}}(\tau)}.
$
The controller compares modes using their recorded rewards, weighted by relevance and recency.

\textit{Cost-adjusted reward.}
Each attempt receives a reward that combines task outcome with additional synthesis cost,
$
r_i = \operatorname{clip}(y_i-\lambda\kappa_i,-1,1),
$
where $\kappa_i$ counts actual extra synthesis calls, including zero when synthesis is skipped. The coefficient $\lambda\ge0$ controls the penalty, which applies only to synthesis calls.

\textit{Relevance weighting.}
Each retrieved entry receives a weight based on task similarity and recency,
$
w_i = \max(\mathrm{sim}(\tau_i,\tau),0)\cdot \gamma^{t_{\max}-t_i},
$
where $t_{\max}$ is the latest timestamp in the retrieved neighborhood and $\gamma\in(0,1]$ controls the weight of older feedback. Anchoring at $t_{\max}$ rescales all $\gamma^{k-t_i}$ weights equally and leaves normalized estimates unchanged. Decay affects only mode-selection feedback, without changing trajectory retrieval. An older mode evaluation may become less representative as memory grows, even when its trajectory remains a useful example.

\textit{Local utility estimation.}
For each mode, the controller averages the rewards of retrieved attempts that used it. If their total weight is positive, $\sum_{i:\mu_i=\mu}w_i>0$, the estimate is
\[
\hat{r}(\mu\mid\tau)
=
\frac{
\sum_{i\in\mathcal{I}_{\hat{B}}(\tau):\,\mu_i=\mu}
w_i r_i
}{
\sum_{i\in\mathcal{I}_{\hat{B}}(\tau):\,\mu_i=\mu}
w_i
}.
\]
This estimate summarizes the mode's observed performance and synthesis cost on similar tasks.

\textit{Mode selection.}
The controller adds an exploration bonus to this estimate and selects the highest-scoring mode:
\begin{align}
\mu^\star
&=
\arg\max_{\mu\in\mathcal{Z}}
\Bigg(
\hat{r}(\mu\mid\tau)
+
c \sqrt{
\frac{\log |\mathcal{N}_{\hat{B}}(\tau)|}{\max(n_\mu,1)}
}
\Bigg),
\end{align}
where
$
n_\mu
=
\left|
\{i\in\mathcal{I}_{\hat{B}}(\tau):\mu_i=\mu\}
\right|.
$
Here, $n_\mu$ counts retrieved attempts using mode $\mu$. For $c>0$, the bonus encourages revisiting modes with fewer local observations. This UCB-inspired rule balances observed reward with exploration. It carries no confidence-bound or regret guarantee for weighted, nonstationary observations. Algorithm~\ref{alg:synth} describes the local branch with a nonempty trusted neighborhood and positive weight for every compared mode. Unsupported means are undefined; initialization, fallback, and exploration overrides use separate branches (\cref{app:method:hyperparams}).

\textit{Execution and feedback.}
After selection, the agent retrieves mode-appropriate trajectories, uses them directly or synthesizes temporary guidance, and executes the task. During training, execution records enter $M_k$ and eligible feedback enters $\mathcal L_k$. Guidance is discarded after use; evaluation leaves both stores unchanged. New records change subsequent estimates and choices. The backbone, mode definitions, synthesis prompts, retrieval procedures, and selection rule remain fixed.

\begin{table}[t]
\centering
\caption{\small Held-out task success rate (\%, mean $\pm$ std over 3 seeds) on ALFWorld, WebShop, and ScienceWorld with three backbones. Anti-persistence is compared with agent baselines and persistent-memory variants.}
\vspace{5pt}
\label{tab:head-to-head}
\small
\renewcommand{\arraystretch}{1.2}
\setlength{\tabcolsep}{4pt}
\begin{tabular*}{\linewidth}{@{\extracolsep{\fill}}lccccccccc@{}}
\toprule
 & \multicolumn{3}{c}{Llama-4-Scout} & \multicolumn{3}{c}{Qwen3-Coder-Next} & \multicolumn{3}{c}{Gemma-4-31B-it} \\
\cmidrule(lr){2-4} \cmidrule(lr){5-7} \cmidrule(lr){8-10}
Method & ALF & Web & Sci & ALF & Web & Sci & ALF & Web & Sci \\
\midrule
ReAct          & 29.3$\pm$3.1 & 33.3$\pm$1.2 & 9.4$\pm$1.0  & 70.4$\pm$1.8 & 36.0$\pm$2.0 & 27.2$\pm$2.6 & 72.7$\pm$1.2 & 43.3$\pm$2.3 & 44.2$\pm$1.2 \\
Reflexion      & 46.7$\pm$6.1 & 45.3$\pm$4.2 & 22.8$\pm$1.0 & 87.3$\pm$3.1 & 47.3$\pm$1.2 & 33.9$\pm$3.5 & \textcolor{blue}{\textbf{98.0$\pm$0.8}} & 50.0$\pm$2.0 & 44.5$\pm$2.6 \\
ExpeL          & 7.2$\pm$3.2  & 38.0$\pm$5.3 & 5.0$\pm$1.9  & 88.0$\pm$2.0 & 47.3$\pm$2.3 & 32.8$\pm$4.8 & 96.7$\pm$2.3 & 38.6$\pm$1.9 & 51.3$\pm$1.2 \\
AutoRefine     & 36.3$\pm$3.2 & 37.3$\pm$2.3 & 15.3$\pm$1.7 & 83.7$\pm$6.7 & 40.7$\pm$2.3 & 27.8$\pm$1.0 & 88.0$\pm$4.0 & 48.0$\pm$4.0 & 43.3$\pm$0.0 \\
\midrule
Traj. Mem.   & 41.6$\pm$1.7 & 45.3$\pm$1.2 & 20.0$\pm$1.7 & 84.4$\pm$2.6 & 50.7$\pm$1.2 & 43.8$\pm$3.7 & 96.0$\pm$1.0 & 48.0$\pm$0.2 & 51.7$\pm$1.7 \\
Abs. Mem.    & 42.7$\pm$1.2 & 37.3$\pm$1.2 & 18.9$\pm$1.0 & 84.7$\pm$2.4 & 38.3$\pm$3.5 & 34.0$\pm$1.0 & 92.7$\pm$2.3 & 45.3$\pm$4.2 & 46.7$\pm$1.7 \\
Joint Mem.   & 33.3$\pm$2.3 & 35.3$\pm$1.2 & 7.2$\pm$1.9  & 81.9$\pm$0.2 & 47.0$\pm$1.7 & 36.3$\pm$0.6 & 95.3$\pm$3.1 & 51.3$\pm$1.2 & \textcolor{blue}{\textbf{56.3$\pm$2.4}} \\
\midrule
\textbf{Anti-pers.} & \textcolor{blue}{\textbf{54.7$\pm$2.9}} & \textcolor{blue}{\textbf{46.3$\pm$1.8}} & \textcolor{blue}{\textbf{30.6$\pm$1.0}} & \textcolor{blue}{\textbf{96.0$\pm$1.4}} & \textcolor{blue}{\textbf{54.0$\pm$2.0}} & \textcolor{blue}{\textbf{52.6$\pm$1.5}} & 97.3$\pm$1.2 & \textcolor{blue}{\textbf{53.3$\pm$3.1}} & 54.1$\pm$6.3 \\
\bottomrule
\end{tabular*}
\end{table}
\section{Experiments}
We evaluate held-out success and cost (\cref{subsec:main-results}), performance across memory updates (\cref{subsec:cross-iteration-analysis}), and individual design choices (\cref{subsec:ablation}). Provenance associations (\cref{app:full-results:descendant}) and controlled content and retention studies (\cref{app:additional-analyses}) provide complementary evidence about the role of inferred guidance.

\subsection{Experimental Setup}
\label{subsec:setup}

\textbf{Benchmarks.}
We evaluate anti-persistence on ALFWorld~\citep{shridhar2020alfworld}, WebShop~\citep{yao2022webshop}, and ScienceWorld~\citep{wang2022scienceworld}, using fixed training and held-out task partitions across methods and iterations following recent agent-memory work~\citep{zhao2024expel,chen2024automanual,shinn2023reflexion,du2026memory}. Partition sizes are reported in \cref{tab:task-counts}.

\textbf{Backbones.}
The main experiments use Llama-4-Scout \citep{meta2025llama4}, Qwen3-Coder-Next \citep{cao2026qwen3}, and Gemma-4-31B-it \citep{deepmind2026gemma4}, spanning distinct model families and capabilities to examine how memory effects vary across backbones.

\textbf{Baselines.}
We compare four memory configurations: \textbf{trajectory memory} retrieves stored trajectories, \textbf{abstraction memory} retrieves stored inferred abstractions, \textbf{joint memory} retrieves both, and \textbf{anti-persistence} constructs temporary guidance from stored trajectories. We also evaluate four agent baselines: ReAct \citep{yao2022react}, Reflexion \citep{shinn2023reflexion}, ExpeL \citep{zhao2024expel}, and AutoRefine \citep{qiu2026autorefine}.

\textbf{Protocol.}
Each method undergoes $K = 10$ iterations with fixed backbone parameters, with memory and related state updated exclusively from training episodes. After each iteration, we evaluate the full held-out set with this state frozen: retrieval is permitted, but stored trajectories, abstractions, outcome logs, retrieval metadata, maintenance scores, and pruning decisions remain unchanged, and evaluation logs never enter subsequent updates. Under a shared per-task context budget, the base trajectory-retrieval budget is $B = 5$, with one additional successful trajectory available as a contrast for corrective mode (\cref{app:method:similarity}). Anti-persistence uses up to $\hat{B}=10$ outcome-log entries for mode selection.

\textbf{Metrics.}
We report held-out success over 3 seeds (\{42,43,44\}), with error bars denoting standard deviations across seeds, and measure cost as the mean number of synthesis and execution LLM calls per held-out task.

\subsection{Held-Out Performance and LLM-Call Cost}
\label{subsec:main-results}
\Cref{tab:head-to-head} compares held-out success at $K=10$ against four agent baselines and three persistent-memory variants. \Cref{fig:cost-summary} reports average per-task LLM-call cost for the four memory configurations.

\begin{figure}[t]
\centering
\begin{subfigure}{\linewidth}
  \centering
  \includegraphics[width=\linewidth]{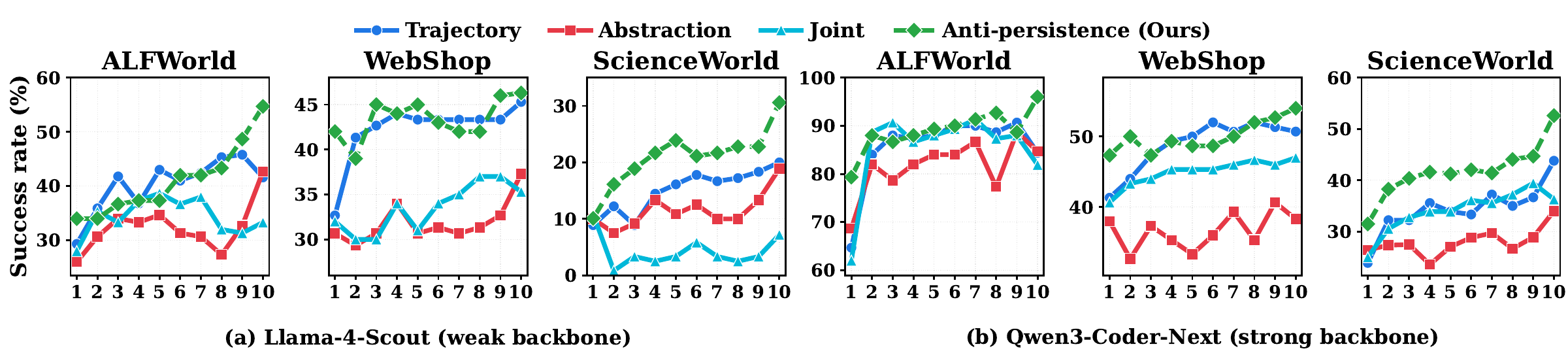}
  \vspace{-15pt}
  \caption{\small Per-iteration held-out task success rate.}
  \label{fig:divergence}
\end{subfigure}
\par\vspace{4pt}
\begin{subfigure}{\linewidth}
  \centering
  \includegraphics[width=\linewidth]{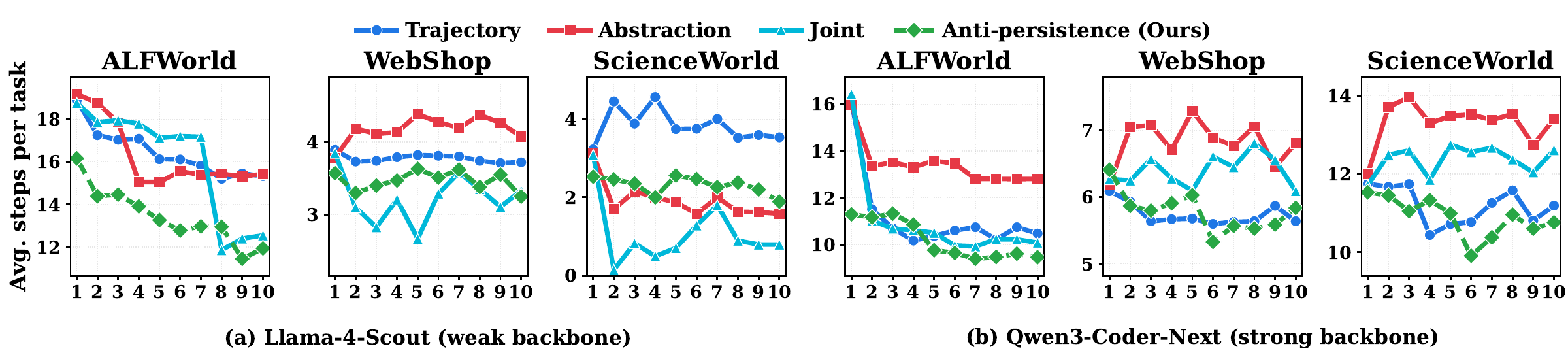}
  \vspace{-15pt}
  \caption{\small Per-iteration mean trajectory length.}
  \label{fig:avg-steps}
\end{subfigure}
\vspace{-15pt}
\caption{\small Per-iteration performance of four memory configurations over ten memory-update iterations across three benchmarks and two backbones.}
\label{fig:divergence-pair}
\end{figure}

\textbf{Anti-persistence improves overall held-out success.} Anti-persistence achieved the highest success rate in seven of the nine benchmark--backbone settings (\Cref{tab:head-to-head}). Averaged equally across settings, it reached $59.9\%$, compared with $53.5\%$ for trajectory memory, $49.0\%$ for abstraction memory, and $49.3\%$ for joint memory. It achieved higher mean success than trajectory and abstraction memory in all nine settings, and joint memory in eight. It also exceeded Reflexion, the strongest agent baseline on average, which reached $52.9\%$. The memory variants further clarify the role of stored abstractions: both abstraction memory and joint memory underperformed trajectory memory in seven of the nine settings. Retaining inferred guidance, either alone or alongside execution records, therefore did not consistently improve success. These comparisons favor anti-persistence overall while leaving the preferred design setting-dependent. Section~\ref{subsec:cross-iteration-analysis} examines how these differences develop across updates.

\begin{wrapfigure}[12]{r}{0.42\linewidth}
\centering
\vspace{-14pt}
\includegraphics[width=\linewidth]{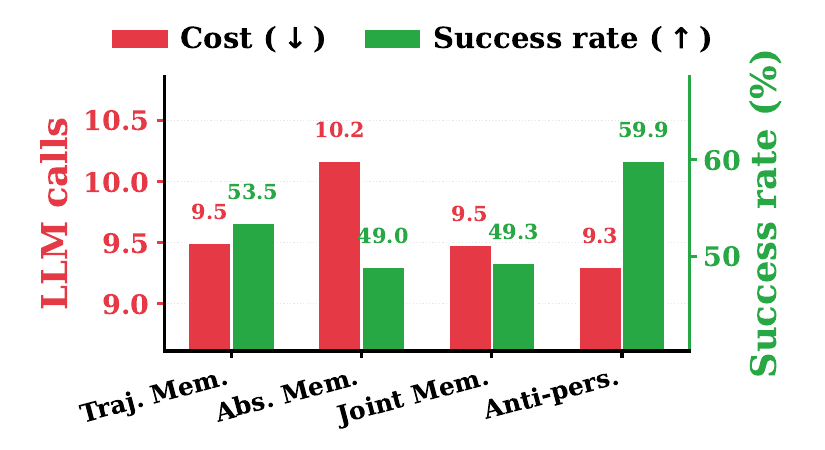}
\vspace{-10pt}
\caption{\small Mean LLM-call cost and held-out success across nine settings.}
\label{fig:cost-summary}
\vspace{-6pt}
\end{wrapfigure}

\textbf{Anti-persistence has the lowest average LLM-call cost.} Across the nine settings, anti-persistence achieved the lowest mean LLM-call cost among the four memory configurations, alongside the highest mean success (\Cref{fig:cost-summary}). Generating temporary guidance therefore did not translate into a higher average call count. Total call usage reflects both guidance synthesis and task execution. The $\mathrm{none}$ mode reuses retrieved trajectories directly without a synthesis call, while shorter execution trajectories in some settings offset calls used to construct guidance. These results support the efficiency of direct experience reuse and selective synthesis.

\subsection{Dynamics Across Self-Improvement Rounds}
\label{subsec:cross-iteration-analysis}
To characterize behavior across memory updates, we first examine how the four memory configurations perform over successive rounds (\Cref{fig:divergence-pair}). We then compare the final results of abstraction memory, AutoRefine, and anti-persistence (\Cref{fig:maintenance}).

\textbf{Memory Behavior Across Iterations.}
Memory effects varied across backbones and benchmarks (\Cref{fig:divergence-pair}). Anti-persistence achieved the highest final mean success in all six settings, although it did not lead at every iteration. Qwen's abstraction-memory performance fluctuated more on WebShop than on ALFWorld, while Llama's joint-memory performance declined in later rounds, particularly on ScienceWorld. These trajectories show that combining trajectory and abstraction memory did not consistently improve performance. Trajectory length provides complementary context (\Cref{fig:avg-steps}), but shorter rollouts can reflect either efficient execution or early failure. Their coincidence with lower success under joint memory is consistent with premature termination, without establishing its cause.

\textbf{Comparison with Memory Maintenance.}
\Cref{fig:maintenance} presents the final results for abstraction memory, AutoRefine, and anti-persistence from \cref{tab:head-to-head}. AutoRefine improved mean success over abstraction memory in one of the six settings, matched it in one, and performed worse in four. Anti-persistence achieved higher mean success than AutoRefine in all six. AutoRefine also changes retrieval and synthesis (\cref{app:setup-details:autorefine}), so this comparison does not isolate maintenance effects or rule out more effective policies. Deleting an abstraction leaves stored trajectories unchanged unless they are explicitly revised (Remark~\ref{rem:parent-deletion}). Maintenance can nevertheless affect subsequent exposure, flagged generation, and persistent admission (Proposition~\ref{prop:residual-risk}).

\begin{figure}[!t]
\centering
\captionsetup{skip=4pt}
\begin{minipage}[t]{0.48\linewidth}
  \vspace{0pt}
  \begin{minipage}[c][100pt][c]{\linewidth}
  \centering
  \includegraphics[width=\linewidth,height=98pt,keepaspectratio]{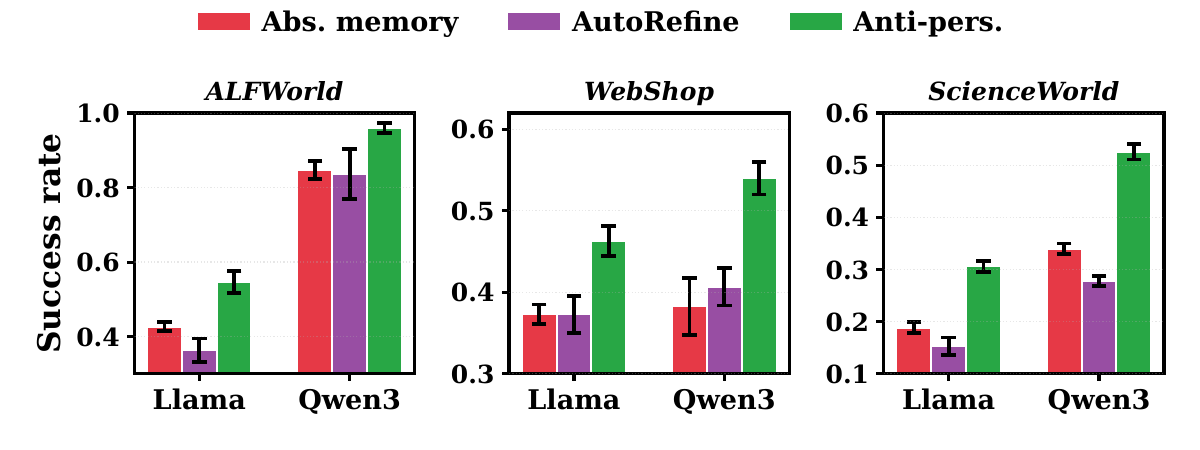}
  \end{minipage}
\caption{\small Final held-out success for abstraction memory, AutoRefine, and anti-persistence from \cref{tab:head-to-head}. Bars show means over 3 seeds; error bars show standard deviations.}
  \label{fig:maintenance}
\end{minipage}\hfill
\begin{minipage}[t]{0.48\linewidth}
  \vspace{0pt}
  \begin{minipage}[c][100pt][c]{\linewidth}
  \centering
  \includegraphics[width=\linewidth,height=98pt,keepaspectratio]{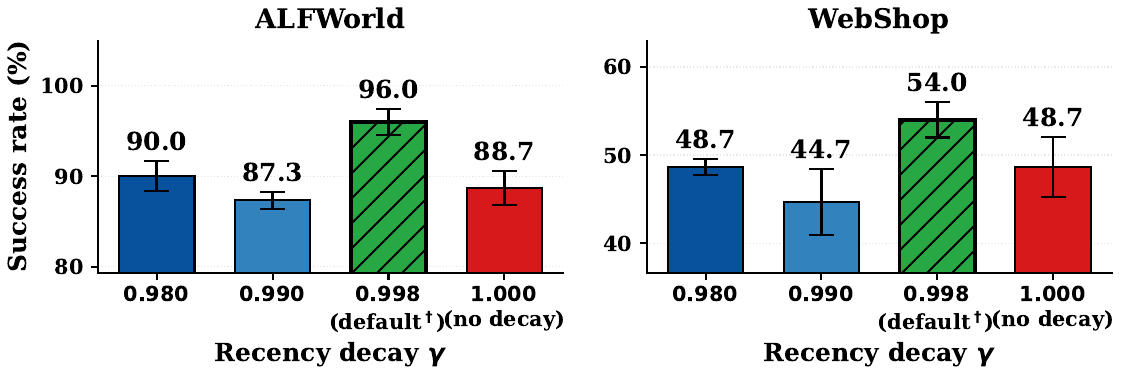}
  \end{minipage}
  \caption{\small Sensitivity to recency decay $\gamma$ on Qwen3-Coder-Next. Bars show means over $3$ seeds; error bars show standard deviation across seeds.}
  \label{fig:sensitivity-gamma}
\end{minipage}
\par\vspace{12pt}
\begin{minipage}[t]{0.48\linewidth}
  \vspace{0pt}
  \begin{minipage}[c][100pt][c]{\linewidth}
  \centering
  \includegraphics[width=\linewidth,height=98pt,keepaspectratio]{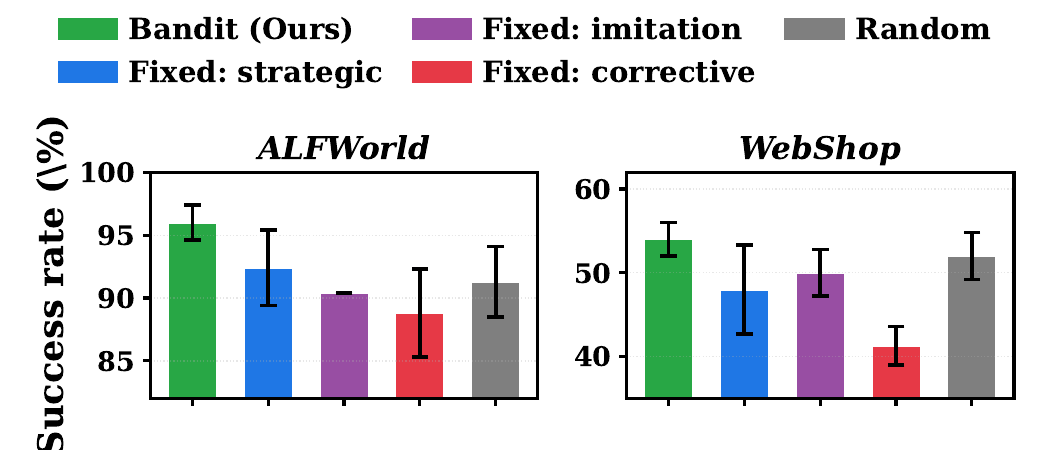}
  \end{minipage}
  \caption{\small Mode-selection ablation on Qwen3-Coder-Next, comparing anti-persistence's contextual bandit with fixed-mode and random policies.}
  \label{fig:ablation-mode}
\end{minipage}\hfill
\begin{minipage}[t]{0.48\linewidth}
  \vspace{0pt}
  \begin{minipage}[c][100pt][c]{\linewidth}
  \centering
  \includegraphics[width=\linewidth,height=98pt,keepaspectratio]{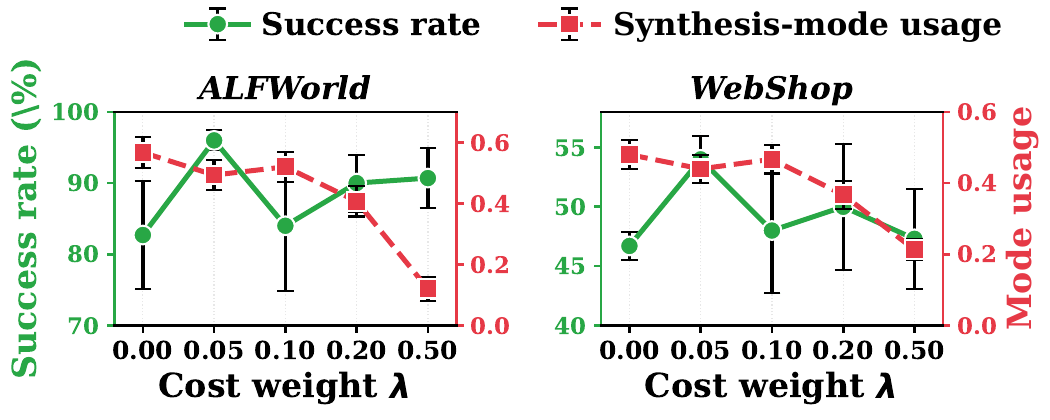}
  \end{minipage}
\caption{\small Sensitivity to $\lambda$ on Qwen3-Coder-Next. Solid curves use the left axis for success rate; dashed curves use the right axis for synthesis-mode usage (fraction of tasks selecting a synthesis mode).}
  \label{fig:ablation-lambda}
\end{minipage}
\end{figure}

\subsection{Ablations and Design Analysis}
\label{subsec:ablation}

\textbf{Sensitivity to Recency Decay $\gamma$.}
The recency factor $\gamma$ controls how quickly older outcome-log feedback loses weight during mode selection. Smaller values discount older feedback more strongly, whereas $\gamma=1$ removes temporal decay while retaining task-similarity weighting. \Cref{fig:sensitivity-gamma} compares four values on Qwen3-Coder-Next. The default $\gamma=0.998$ achieved the highest mean success among the tested values: $96.0\%$ on ALFWorld and $54.0\%$ on WebShop. At $\gamma=0.980$, success was $90.0\%$ on ALFWorld and $48.7\%$ on WebShop; at $\gamma=0.990$, it was $87.3\%$ and $44.7\%$, respectively. Without decay ($\gamma=1$), success was $88.7\%$ and $48.7\%$. Both faster decay and no decay therefore underperformed the default. These results favor retaining historical feedback while modestly emphasizing recent outcomes, although success did not vary monotonically with $\gamma$.

\textbf{Benefits of Adaptive Mode Selection.}
To assess the value of adapting mode selection to tasks and feedback, \Cref{fig:ablation-mode} compares the contextual bandit with fixed strategic, imitation, and corrective policies and with random selection on Qwen3-Coder-Next. The bandit achieved the highest success rate on both ALFWorld ($96.0\%$) and WebShop ($54.0\%$). The strongest fixed policy differed by benchmark: strategic performed best on ALFWorld and imitation on WebShop. The bandit's advantage over fixed and random selection supports using task context and accumulated feedback to decide how retrieved experience should guide execution.

\textbf{Sensitivity to Cost Weight $\lambda$.}
The cost weight $\lambda$ penalizes additional synthesis calls in the mode-selection reward; $\lambda=0$ removes this penalty. \Cref{fig:ablation-lambda} compares success and synthesis-mode usage across cost weights. The default $\lambda=0.05$ achieved the highest mean success among the tested values on both ALFWorld ($96.0\%$) and WebShop ($54.0\%$), consistent with \cref{tab:head-to-head}. Removing the penalty increased synthesis-mode usage but reduced mean success relative to the default; stronger penalties ($\lambda\in\{0.20,0.50\}$) reduced both. These results favor selective synthesis with a modest penalty in the tested settings. Usage is the fraction of tasks selecting a synthesis mode, even when missing records prevent an actual call.

\textbf{Mode Selection Patterns Across Settings.}
To examine how the controller allocated modes across settings, \Cref{fig:bandit-mode-freq} reports selection frequencies averaged over iterations for nine $(\text{backbone}\times\text{benchmark})$ combinations. All settings used multiple modes. On WebShop, stronger backbones favored $\mathrm{imitation}$ and Llama-4-Scout favored $\mathrm{strategic}$, while ALFWorld and ScienceWorld showed more balanced usage. The $\mathrm{none}$ mode, which reuses retrieved trajectories without synthesis, had a nonzero mean frequency in every setting, so the controller sometimes reused experience directly. These averages describe overall allocation, without showing whether a mode was selected in every round or establishing policy convergence.

\begin{table}[!htbp]
\centering
\caption{\small Held-out success with abstractions persisted or discarded (\%, mean $\pm$ s.d.; 3 seeds).}
\label{tab:discard-vs-persist}
\small
\setlength{\tabcolsep}{3pt}
\renewcommand{\arraystretch}{1.1}
\begin{tabular*}{0.85\linewidth}{@{\extracolsep{\fill}}llccc@{}}
\toprule
Model & Bench & Discard (Ours) & Persist & $\Delta$ \\
\midrule
Llama-4-Scout    & ALF & \textbf{54.7$\pm$2.9} & 26.7$\pm$6.1 & +28.0 \\
Llama-4-Scout    & Web & \textbf{46.3$\pm$1.8} & 34.0$\pm$4.0 & +12.3 \\
Qwen3-Coder-Next & ALF & \textbf{96.0$\pm$1.4} & 92.7$\pm$1.2 & \phantom{0}+3.3 \\
Qwen3-Coder-Next & Web & \textbf{54.0$\pm$2.0} & 47.3$\pm$3.1 & \phantom{0}+6.7 \\
\bottomrule
\end{tabular*}
\end{table}

\textbf{Persisting versus Discarding Abstractions.}
To assess abstraction persistence, we compared variants that used synthesized guidance for the current task when synthesis was invoked, but differed in whether it was subsequently discarded or stored for later retrieval. Discarding improved success in all four settings (\Cref{tab:discard-vs-persist}), with larger gains for Llama-4-Scout and positive gains for Qwen3-Coder-Next. These comparisons favor using synthesized guidance temporarily, without storing it for later tasks, in the evaluated settings.

\begin{figure}[t]
\centering
\begin{subfigure}{0.32\linewidth}
  \centering
  \includegraphics[width=\linewidth]{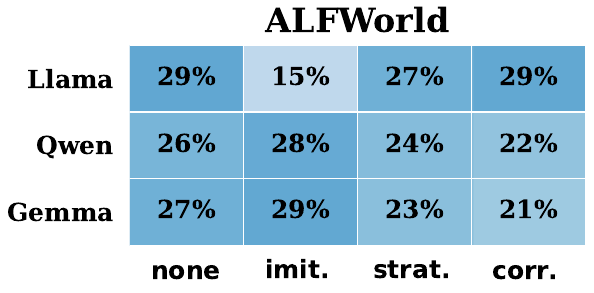}
  \label{fig:bandit-mode-freq-alfworld}
\end{subfigure}
\hfill
\begin{subfigure}{0.32\linewidth}
  \centering
  \includegraphics[width=\linewidth]{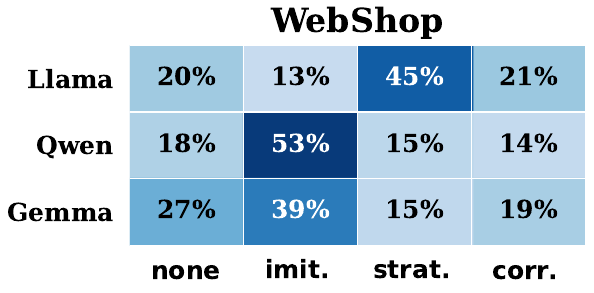}
  \label{fig:bandit-mode-freq-webshop}
\end{subfigure}
\hfill
\begin{subfigure}{0.32\linewidth}
  \centering
  \includegraphics[width=\linewidth]{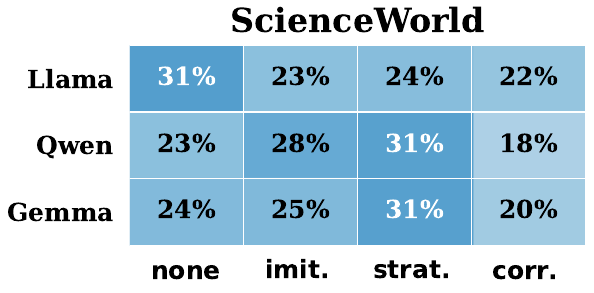}
  \label{fig:bandit-mode-freq-scienceworld}
\end{subfigure}
\vspace{-10pt}
\caption{\small Mode-selection frequencies averaged over iterations $1$--$10$. Panels denote benchmarks, rows denote backbones, and columns denote modes ($\mathrm{none}$, imitation, strategic, corrective). Darker cells indicate more frequent selection.}
\label{fig:bandit-mode-freq}
\end{figure}

\section{Conclusion}

We introduced anti-persistence for memory-based self-improvement: agents accumulate trajectories and outcomes and adaptively construct temporary guidance with a fixed backbone and synthesis procedure. The design removes direct storage and reuse of synthesized abstractions, while indirect feedback through trajectories remains possible. Our analysis distinguishes exposure, flagged generation, and persistent admission without implying that all maintenance policies fail. Across three benchmarks and three backbones, anti-persistence achieved higher overall held-out success with fewer LLM calls on average. These findings support separating experience accumulation from abstraction persistence and studying when verification can justify selective retention.

\bibliography{references}
\bibliographystyle{unsrtnat}

\clearpage
\appendix
\onecolumn
\renewcommand{\topfraction}{0.9}
\renewcommand{\bottomfraction}{0.85}
\renewcommand{\textfraction}{0.1}
\renewcommand{\floatpagefraction}{0.8}
\setlength{\parskip}{4.5pt plus 1pt minus 0.5pt}

\section*{Contents of the Appendix}
\startcontents[appendix]

\begingroup
  \printcontents[appendix]{app}{1}[2]{%
    \setlength{\parskip}{0pt}%
    \contentsmargin{20pt}%
  }
\endgroup

\clearpage

\section{Reproduction Dynamics in Abstraction Memory}
\label{app:dynamics}

This appendix formalizes the relationship between historical exposure to stored abstractions and the generation and retention of flagged candidates by distinguishing provenance, operational bias labels, and admission to persistent memory. Each abstraction symbol denotes a unique, immutable object identifier to which labels and parent sets refer; rewriting or merging creates a new identifier. The results concern a single synthesis invocation under stated conditions and establish neither progressive improvement nor inevitable long-run degradation, nor do they imply that all maintenance policies fail.

\subsection{Operational Bias Labels and Historical Provenance}
\label{app:dynamics-defs}

\begin{definition}[Operational bias labels]
\label{def:biased}
Fix an evaluation and labeling protocol $\mathcal E$. Let $\mathcal S_{\mathrm{lab}}$ be the abstractions with defined labels under that protocol, and let $b_{\mathcal E}:\mathcal S_{\mathrm{lab}}\to\{0,1\}$. Define the flagged set
\[
\mathcal B_{\mathcal E}=\{s\in\mathcal S_{\mathrm{lab}}:b_{\mathcal E}(s)=1\}.
\]
An abstraction in $\mathcal B_{\mathcal E}$ is called operationally biased. Abstractions outside $\mathcal S_{\mathrm{lab}}$ are unclassified; absence of a flag does not certify that an abstraction is benign.
\end{definition}

The evaluation protocol and resulting labels remain fixed within each theoretical comparison, so all probability statements below concern operational flags rather than an unobserved intrinsic error rate. Our empirical protocol assigns these labels using associations with downstream success (\cref{app:full-results:descendant}), which do not identify the causal effect of an abstraction.

\paragraph{Historical provenance.}
For a stored execution record $m_j=(\tau_j,e_j,y_j)$, let $P(m_j)$ contain the identifiers of persistent abstractions retrieved into the agent's context when $e_j$ was generated. These identifiers remain attached to the record for analysis even after the corresponding abstractions are deleted, documenting historical exposure without by themselves establishing behavioral influence.

For a synthesis input $H\subseteq M_k$, define its historical parent set and flagged-parent exposure event by
\[
P(H)=\bigcup_{m\in H}P(m),
\qquad
C=\{P(H)\cap\mathcal B_{\mathcal E}\ne\emptyset\}.
\]
Execution records and abstraction identifiers remain distinct types throughout this analysis, and every parent in $P(m_j)$ must have existed when the record was generated, so parent sets never depend on future trajectories. We call $C$ an ``exposed input'' event without assuming that the input inherited any particular error from a parent.

For each synthesized candidate $s=\phi(H)$, record $P(s)=P(H)$ at creation. A directed edge $u\to s$ exists when $u\in P(s)$; descendants are objects reachable along these recorded synthesis edges, excluding the starting object itself.

\subsection{Exposure, Generation, and Persistent Admission}
\label{app:two-modes}

Consider a single synthesis invocation conditional on a fixed post-maintenance state, with its probability law encompassing task and mode selection, input retrieval, and synthesis randomness. Writing the candidate output as $s=\phi(H)$ with other conditioning arguments suppressed, define
\[
A=\{s\in\mathcal B_{\mathcal E}\},
\qquad
W=\{s\text{ is admitted as the same object to persistent abstraction memory}\}.
\]
Rewriting or merging a candidate creates a new object requiring its own label and admission analysis; all probabilities below refer to the same fixed state and probability law for the invocation.

\begin{definition}[Conditional reproduction coefficient]
\label{def:alpha}
Let $q=\Pr(C)$. If $q>0$, define
\[
\alpha=\Pr(A\mid C).
\]
If $\Pr(A\cap C)>0$, also define the conditional admission probability
\[
\eta=\Pr(W\mid A\cap C).
\]
The coefficient $\alpha$ is the flagged-output rate conditional on flagged-parent exposure. It is not an offspring count, a causal effect, or a growth factor.
\end{definition}

These quantities may vary with the state, retrieval policy, synthesis operator, and maintenance procedure. A positive $\alpha$ alone establishes neither an increase in the flagged-output rate relative to other inputs nor the accumulation of errors across successive iterations.

\subsection{What Deletion and Maintenance Do Establish}
\label{app:dynamics-propositions}

\begin{remark}[Scope of parent deletion]
\label{rem:parent-deletion}
Immediately before a deletion-only operation, let $S^-$ be the abstraction store and $D^-\subseteq S^-$ its currently retained descendants of a parent $s_0$. Descendants are identified by the historical provenance relation above, possibly over multiple generations, and exclude $s_0$ itself. If the operation deletes $R\subseteq S^-$ and does not alter execution records or their provenance, then
\[
S^+=S^-\setminus R,\qquad D^+=D^-\setminus R,\qquad M^+=M^-.
\]
In particular, deleting only $s_0$ leaves $D^-$ unchanged, whereas deleting descendants can remove them. Removing a parent does not automatically erase its recorded history, but the remaining influence need not be irreversible.
\end{remark}

These identities follow directly from the specified deletion operation, without implying that descendants survive other maintenance steps or that retained records are necessarily retrieved again.

\begin{proposition}[One-step generation and retention of flagged abstractions]
\label{prop:residual-risk}
Under the common law specified above, suppose $q>0$. Then
\[
\Pr(A\cap C)=q\alpha,\qquad \Pr(A)\ge q\alpha.
\]
If $q\alpha>0$, then
\[
\Pr(A\cap C\cap W)=q\alpha\eta,
\qquad
\Pr(A\cap W)\ge q\alpha\eta.
\]
Thus, positive $q$, $\alpha$, and $\eta$ imply a positive probability of generating and admitting a flagged candidate with flagged-parent exposure.
\end{proposition}

\begin{proof}
Since $q>0$, the definition of conditional probability gives
$\Pr(A\cap C)=\Pr(C)\Pr(A\mid C)=q\alpha$.
The inclusion $A\cap C\subseteq A$ gives the first inequality. If $q\alpha>0$, conditioning on $A\cap C$ yields
$\Pr(A\cap C\cap W)=\Pr(A\cap C)\Pr(W\mid A\cap C)=q\alpha\eta$.
The inclusion $A\cap C\cap W\subseteq A\cap W$ gives the second inequality. No independence assumption is required.
\end{proof}

When $q=0$, flagged-parent exposure has zero probability and $\alpha$ is undefined; when $q>0$ but $\alpha=0$, generating a flagged candidate with flagged-parent exposure has zero probability and $\eta$ is undefined. Neither case justifies assigning a value to the undefined conditional quantity.

\paragraph{Scope of the result.}
Proposition~\ref{prop:residual-risk} provides a probability decomposition rather than a positive lower bound valid across all maintenance policies, because maintenance may change future exposure $q$, the conditional input or synthesis distribution determining $\alpha$, and the admission probability $\eta$. For example, excluding all exposed records sets $q=0$, while rejecting every flagged candidate prevents flagged candidates from entering persistent memory. By separating historical exposure, flagged generation, and persistent admission within a single synthesis invocation, the proposition identifies distinct components of the process without establishing inevitable amplification or a long-horizon performance ordering.

\subsection{Implications for Anti-Persistence}
\label{app:dynamics:anti-persistence}

Because anti-persistence admits no synthesized candidate to the abstraction store, $W$ is impossible and $\Pr(A\cap W)=0$ for every task-time synthesis invocation. This storage property does not establish $\Pr(A)=0$, since temporary abstractions may still be flagged and can influence stored trajectories, subsequent synthesis, and mode selection.

The historical exposure event $C$ tracks previously stored abstractions and therefore excludes feedback arising from temporary guidance. In a fresh anti-persistence run with no stored ancestors, $q=0$ and the corresponding $\alpha$ is undefined, not zero. Both raw trajectories and abstractions can carry the effects of past decisions into later tasks, with the extent of this influence determined by retrieval and use rather than memory type alone. The structural guarantee is consequently limited to the absence of stored abstractions and their direct reuse, while the performance benefit remains an empirical finding.

\section{Method Details}
\label{app:method}

This appendix specifies the abstraction modes, prompts, retrieval procedures, and hyperparameters used by anti-persistence, together with the outcome log that connects observed task results to subsequent mode selection.

\subsection{Abstraction Modes}
\label{app:method:modes}

The mode set $\mathcal{Z} = \{\mathrm{none}, \mathrm{imitation}, \mathrm{strategic}, \mathrm{corrective}\}$ provides direct access to factual context or one of three forms of on-demand abstraction, as introduced in \cref{sec:method}. While $\mathrm{none}$ passes retrieved records directly to the agent, $\mathrm{imitation}$ extracts reusable behavioral patterns, $\mathrm{strategic}$ develops higher-level task strategies, and $\mathrm{corrective}$ generates guidance for avoiding observed failures. When the required records are available, the three synthesis modes invoke $\phi$ at task time and discard its output after execution, using the input and output conventions in \cref{app:method:prompts}.

\paragraph{$\mu = \mathrm{none}$ (direct factual context).}
The agent receives factual context $H_\tau$, comprising retrieved trajectory records and observed outcomes, directly without a synthesis call, so $\sigma_\tau = \phi(H_\tau,\mathrm{none}) = H_\tau$ and $\kappa(\mathrm{none}) = 0$. When no records are available, including at initialization or after an embedding failure, $H_\tau = \emptyset$ and the mode reduces to the underlying ReAct policy.

\paragraph{$\mu = \mathrm{imitation}$ (behavioral-pattern abstraction).}
$\phi$ extracts reusable behavioral patterns from successful trajectories, including action--observation sequences, tool-use routines, and their applicability conditions, to provide task-specific guidance without copying a trajectory verbatim. This mode emphasizes reusable action patterns, whereas $\mathrm{strategic}$ focuses on higher-level sub-goal organization and planning. The resulting abstraction is injected under \texttt{\#\# Behavioral Patterns (imitation)} and discarded after execution, with $\kappa(\mathrm{imitation}) = 1$ when synthesis is invoked; if no successful trajectory is available, $\sigma_\tau = \emptyset$ and no synthesis call is charged.

\paragraph{$\mu = \mathrm{strategic}$ (multi-trace strategy distillation).}
$\phi$ uses up to the top-$B$ retrieved successful trajectories to generate a strategy covering sub-goal structure, tool-use patterns, and recovery heuristics, with a prompt that distinguishes transferable structure from instance-specific details. The resulting guidance is injected under \texttt{\#\# Task-Specific Strategy (strategic reflection)}, incurring $\kappa(\mathrm{strategic}) = 1$ when synthesis is invoked. If no successful trajectories are available, $\sigma_\tau = \emptyset$ and no call is charged.

\paragraph{$\mu = \mathrm{corrective}$ (failure-driven avoidance rules).}
$\phi$ uses up to the top-$B$ retrieved failure trajectories and one contrasting success when available, prompting an LLM to identify failure modes, critical decision points, and recurring error patterns. The output consists of avoidance rules of the form ``do NOT do $X$ because $Y$,'' optionally paired with corrective actions, and is injected under \texttt{\#\# Avoidance Rules (corrective reflection)}. Invoking synthesis incurs $\kappa(\mathrm{corrective}) = 1$.

\textbf{Design choice.} When no failure trajectories are available, $\phi$ returns $\sigma_\tau = \emptyset$ rather than falling back to $\mathrm{strategic}$, preventing strategic guidance from being attributed to the $\mathrm{corrective}$ arm during bandit updates.
\paragraph{Cost convention.}
We define $\kappa(\mu) \geq 0$ as the number of additional LLM calls made by mode $\mu$ relative to direct trajectory reuse, so direct factual context has $\kappa(\mathrm{none}) = 0$. Each synthesis mode ($\mathrm{imitation}$, $\mathrm{strategic}$, or $\mathrm{corrective}$) incurs $\kappa(\mu) = 1$ when invoked and zero when synthesis is skipped because the required records are unavailable. The cost weight $\lambda$ places these call counts on the reward scale used for success outcomes, with defaults reported in \cref{app:method:hyperparams}.

\subsection{Mode-Specific Prompts}
\label{app:method:prompts}

Each synthesis mode in \cref{sec:method} uses a prompt to produce $\sigma_\tau$ and a wrapper to present the guidance under a mode-specific header in the agent's context. We provide the strategic and corrective templates, describe the imitation interface, and specify the wrappers below, including the direct factual-context wrapper used by $\mathrm{none}$ to bypass synthesis.

\paragraph{Strategic synthesis prompt.}
The $\mathrm{strategic}$ mode prompts the synthesis LLM to distill a transferable strategy from up to the top-$B$ retrieved successful trajectories ($B=5$ by default).

\begin{synthprompt}[title={Strategic synthesis prompt}]
\begin{lstlisting}[style=promptlisting]
You are synthesizing a task-specific strategy for an agent.
The agent is about to attempt this task:
  "{instruction}"

Below are complete action-by-action trajectories from similar tasks
that SUCCEEDED. Each trajectory shows the exact tool calls made and
the observations received.

{success_text}
Analyze what the successful trajectories share: the overall approach,
the order of sub-goals, the tool usage patterns, the recovery
heuristics. Distinguish what is TRANSFERABLE (structure, sub-goal
ordering, tool patterns) from what is NOT (specific object names,
specific locations, specific observation text).

Output a concise strategy the agent should follow for the current
task. Include reasoning for each recommendation. Do NOT echo the
trajectories verbatim -- distill them into actionable guidance.
Output plain text only.
\end{lstlisting}
\end{synthprompt}

\noindent The placeholders \texttt{\{instruction\}} and \texttt{\{success\_text\}} contain the current task description and up to the top-$B$ retrieved successful trajectories, respectively, with trajectories concatenated as $\texttt{--- Success } i \texttt{ ---}$ blocks. The synthesis LLM uses temperature $0$ and a $1024$-token budget.

\Needspace{6\baselineskip}
\paragraph{Corrective synthesis prompt.}
The $\mathrm{corrective}$ mode prompts the synthesis LLM to extract avoidance rules from up to the top-$B$ retrieved failure trajectories, using a single successful trajectory as a contrast when available.
\begin{synthprompt}[title={Corrective synthesis prompt}]
\begin{lstlisting}[style=promptlisting]
You are synthesizing avoidance rules for an agent.

The agent is about to attempt this task:
  "{instruction}"

Below are trajectories from similar tasks that FAILED. Study them
carefully -- identify the failure modes, the decision points where
things went wrong, and the recurring mistakes:

{failure_text}

For contrast, here are a few SUCCESSFUL trajectories on similar
tasks. Use them sparingly -- only to show the correct move at the
points where failures went wrong:

{success_text}

Output concise avoidance rules ("do NOT do X because Y") and, where
helpful, the corrective action the successful trajectories took at
the same decision point. Include reasoning. Output plain text only.
\end{lstlisting}
\end{synthprompt}

\noindent The placeholders \texttt{\{instruction\}} and \texttt{\{failure\_text\}} contain the current task description and up to the top-$B$ retrieved failures, respectively. The placeholder \texttt{\{success\_text\}} contains the top-1 successful trajectory when available and otherwise takes the value \texttt{(no successful reference available)}, allowing synthesis to proceed without a positive reference. Decoding settings match those of the strategic prompt.

\paragraph{Imitation synthesis interface.}
Given the current task and up to the top-$B$ retrieved successful trajectories, $\mathrm{imitation}$ identifies successful action patterns, their supporting observations, and ways to adapt them to the current task without generalizing instance-specific object names or locations into rules. The synthesized guidance is presented through the abstraction wrapper below rather than as a verbatim copy of a reference trajectory.

\paragraph{Direct factual-context wrapper ($\mathrm{none}$).}
The $\mathrm{none}$ mode supplies retrieved records and their observed outcomes without a synthesis call, using the following factual-context header:

\begin{wrapprompt}[title={Direct factual context},breakable=false]
\begin{lstlisting}[style=promptlisting]
## Retrieved Factual Context (none)
Below are records of related past task executions and their outcomes.
Use them as evidence; adapt to the actual observations of this task.

<retrieved trajectory records and observed outcomes>
\end{lstlisting}
\end{wrapprompt}

\noindent Factual records retain their task instructions, observed outcomes, and action--observation sequences within the shared context budget, with formatting that adds no inferred behavioral rules and requires no synthesis LLM call.

\begin{samepage}
\paragraph{Injection wrapper.}
The output of $\mathrm{imitation}$, $\mathrm{strategic}$, or $\mathrm{corrective}$ receives a mode-specific markdown header before being appended to the agent's system prompt.

\begin{wrapprompt}[title={Guidance injection},breakable=false]
\begin{lstlisting}[style=promptlisting]
## <mode-specific header>
The following was synthesized from similar past trajectories for
this specific task. Treat it as guidance -- adapt to the actual
observations you receive.

<synthesized abstraction>
\end{lstlisting}
\end{wrapprompt}
\end{samepage}

\begin{sloppypar}
\noindent The headers \texttt{Behavioral Patterns (imitation)}, \texttt{Task-Specific Strategy (strategic reflection)}, and \texttt{Avoidance Rules (corrective reflection)} identify the source mode alongside each synthesized abstraction. Under $\mathrm{none}$, the agent receives the factual-context wrapper above, which is omitted only when no records are available.
\end{sloppypar}

\subsection{Retrieval Similarity}
\label{app:method:similarity}

Both retrieval pathways operate on a fixed task embedding $\mathbf{e}_\tau \in \mathbb{R}^{d}$ using the cosine similarity function
\[
\mathrm{sim}(\tau_i, \tau) = \frac{\langle \mathbf{e}_{\tau_i}, \mathbf{e}_\tau \rangle}{\lVert \mathbf{e}_{\tau_i}\rVert \cdot \lVert \mathbf{e}_\tau\rVert + \epsilon}, \quad \epsilon = 10^{-12},
\]
but serve distinct purposes: \emph{trajectory retrieval} selects records for direct use or synthesis, whereas \emph{mode-selection retrieval} selects outcome-log neighbors for estimating local bandit rewards. These purposes lead to differences in candidate sets, time decay, and task-type filtering.

\paragraph{Task embedding.}
For each task $\tau$, we compute one embedding by concatenating the task instruction with a truncated prefix of the initial observation:
\[
\texttt{embed\_text}
= \texttt{instruction}
\;\Vert\;
\texttt{"Initial state:"}
\;\Vert\;
\texttt{truncate}(\texttt{obs}, 500).
\]
Including the initial observation helps distinguish tasks that share a goal template but differ in their starting state. Qwen3-Embedding-8B~\citep{zhang2025qwen3embedding} computes this embedding in a single API call; if the call fails, $\mathbf{e}_\tau$ remains unset and the controller defaults to $\mu = \mathrm{none}$ for the current task without an in-iteration retry.

\paragraph{Trajectory retrieval (input to $\phi$).}
Let $R_\mu$ denote retrieval from trajectory memory under the filtering rules and budgets for mode $\mu$. Under $\mathrm{none}$, $R(\tau,M_k,B)$ returns up to $B$ relevant factual records and their outcomes without synthesis or success-only filtering. The $\mathrm{imitation}$ and $\mathrm{strategic}$ modes select up to the top-$B$ successes ($y=1$), whereas $\mathrm{corrective}$ selects up to the top-$B$ failures ($y=0$) and one contrasting success when available. Each candidate set is ranked by task similarity after task-type filtering, with a default budget of $B=5$ (\cref{app:method:hyperparams}).

\paragraph{Mode-selection retrieval (input to the bandit).}
The contextual bandit in \cref{sec:method} searches the entire outcome log $\mathcal{L}_k$ using the same task embedding and cosine similarity, computing $\mathrm{sim}(\tau_i, \tau)$ for every entry in a single matrix--vector product. The top-$\hat B$ entries, or all entries if fewer are available, form the local neighborhood $\mathcal{N}_{\hat B}(\tau)$ used for the cost-adjusted, recency-weighted reward estimates in \cref{sec:method}. If average neighborhood similarity falls below $\rho$, the bandit falls back to a global-best policy that ignores the neighborhood; both $\rho$ and $\hat B$ are reported in \cref{app:method:hyperparams}.

\subsection{Hyperparameters}
\label{app:method:hyperparams}

\Cref{tab:hyperparams} lists the method-level hyperparameters and default settings for anti-persistence, with sensitivity analyses varying only the specified parameters. Symbols follow the main text, supplemented by parameters governing bandit exploration, while \cref{app:setup-details} reports run-level settings such as iteration counts, episodes per iteration, and per-step budgets.

\begin{table}[h]
\centering
\caption{Method hyperparameters and default values. Sensitivity analyses vary the specified parameters.}
\label{tab:hyperparams}
\small
\setlength{\tabcolsep}{4pt}
\renewcommand{\arraystretch}{1.15}
\begin{tabularx}{\linewidth}{@{}p{0.27\linewidth}p{0.21\linewidth}Y@{}}
\toprule
Parameter (symbol) & Default & Description \\
\midrule
\multicolumn{3}{l}{\emph{Bandit: mode selection}} \\
KNN neighborhood size $\hat{B}$         & $10$        & top-$\hat{B}$ entries form $\mathcal{N}_{\hat{B}}(\tau)$ \\
UCB coefficient $c$                     & $0.5$       & exploration bonus base \\
Recency decay $\gamma$                  & $0.998$     & weight $w_i \propto \gamma^{\,t_{\max}-t_i}$ \\
KNN trust threshold $\rho$              & $0.7$       & fallback to global-best if $\overline{\mathrm{sim}} < \rho$ \\
Per-mode warmup $W_{\mathrm{init}}$                     & $3$         & each mode forced at least $W_{\mathrm{init}}$ times before KNN-UCB \\
Round-robin warmup iterations           & $3$         & first $3$ iterations enforce round-robin selection \\
$\varepsilon$-greedy rate $\varepsilon$ & $0.1$       & probability of uniform-random mode selection \\
Reward clipping range                   & $[-1, 1]$   & per-task bandit reward clipped to this range \\
\midrule
\multicolumn{3}{l}{\emph{Reward shaping}} \\
Cost weight $\lambda$                   & $0.05$      & $r=\operatorname{clip}(y-\lambda\kappa,-1,1)$ \\
\midrule
\multicolumn{3}{l}{\emph{Trajectory retrieval (input to $\phi$)}} \\
Success top-$k$ ($B$)                   & $5$         & success trajectories retrieved per call \\
Failure top-$k$ ($B$)                   & $5$         & failure trajectories (corrective mode only) \\
Contrasting success top-$k$             & $1$         & success reference for corrective mode \\
Task-type filter                        & on          & retrieval restricted to the current task type \\
\midrule
\multicolumn{3}{l}{\emph{LLM calls}} \\
Synthesis max tokens                    & $1024$      & for $\phi$ in imitation, strategic, and corrective \\
Agent max tokens                        & $1024$      & per action / step \\
\midrule
\multicolumn{3}{l}{\emph{Embedding}} \\
Embedding model                         & Qwen3-Embedding-8B & one call per task; see \cref{app:method:similarity} \\
Init-observation prefix length          & $500$ chars & characters of $\texttt{obs}$ concatenated into \texttt{embed\_text} \\
\bottomrule
\end{tabularx}
\end{table}

\noindent The following notes explain the design choices underlying the numerical defaults.

\noindent\textit{Round-robin and $\varepsilon$-greedy as safeguards.}
The first three experimental iterations use round-robin selection to provide initial mode coverage, subject to the logging exclusions in \cref{app:method:outcome-log}. After warmup, an $\varepsilon$-greedy rate of $0.1$ maintains exploration by giving every mode a nonzero chance of selection, complementing the initial coverage provided by round-robin selection.

\noindent\textit{KNN trust threshold $\rho = 0.7$.}
When the average cosine similarity within $\mathcal{N}_{\hat{B}}(\tau)$ falls below $\rho$, the bandit switches to its global-best fallback, reducing reliance on weak local matches without guaranteeing accurate global estimates. We set $\rho$ to $0.7$ using a small validation set and perform no further tuning.

\noindent\textit{Reward clipping.}
Per-task bandit rewards $r=\operatorname{clip}(y-\lambda\kappa,-1,1)$ are recorded in the outcome log $\mathcal{L}_k$ using the realized synthesis-call count $\kappa$, which is zero when synthesis is skipped. Although the defaults ($y \in \{0, 1\}$, $\lambda = 0.05$, $\kappa \in \{0, 1\}$) already produce rewards within this range, clipping accommodates potential changes in the reward or cost scale.

\subsection{Outcome-Log Schema and Update Protocol}
\label{app:method:outcome-log}

The outcome log $\mathcal{L}_k$ introduced in \cref{sec:method} records observed mode choices and outcomes without storing synthesized abstractions. Its schema and update protocol are specified below; because the recorded outcomes inform later mode selection, this factual log still permits feedback across tasks.

\paragraph{Record schema.}
Each entry of $\mathcal{L}_k$ is a tuple of seven fields:
\begin{itemize}[leftmargin=2em,itemsep=0pt]
  \item \texttt{embedding} (\texttt{list[float]}): the task embedding $\mathbf{e}_\tau$ from \cref{app:method:similarity}, used by the bandit's KNN lookup;
  \item \texttt{mode} (\texttt{str}): one of $\{\mathrm{none}, \mathrm{imitation}, \mathrm{strategic}, \mathrm{corrective}\}$;
  \item \texttt{reward} (\texttt{float}, clipped to $[-1, 1]$): the cost-adjusted bandit reward $r=\operatorname{clip}(y-\lambda\kappa,-1,1)$;
  \item \texttt{raw\_success} (\texttt{bool}): the binary task outcome $y$, independent of $\lambda$;
  \item \texttt{task\_type} (\texttt{str}): the task family (e.g., \texttt{alfworld});
  \item \texttt{task\_id} (\texttt{int}): the index of $\tau$ in its benchmark task pool;
  \item \texttt{instruction} (\texttt{str}, truncated to $\le 500$ characters): the task instruction text, retained for stratified analysis and debugging.
\end{itemize}

\paragraph{Storage format.}
$\mathcal{L}_k$ is stored as a single append-only JSONL file with one record per line, without rewriting, compaction, or deletion of entries. The example below displays one record with line breaks and abbreviated embedding values for readability:

\begin{wrapprompt}
\begin{lstlisting}[style=promptlisting]
{"embedding": [0.0123, -0.0456, ..., 0.0789],
 "mode": "strategic", "reward": 0.95, "raw_success": true,
 "task_type": "alfworld", "task_id": 7,
 "instruction": "put a clean mug on the coffee table"}
\end{lstlisting}
\end{wrapprompt}
\color{black}

\paragraph{Update protocol.}
After each training-task attempt, the agent appends one record to $\mathcal{L}_k$, subject to three skip rules:
\begin{enumerate}[leftmargin=2em,itemsep=2pt]
  \item \emph{The initial task update $k=0$ is excluded.} At $k = 0$, trajectory memory $M_k$ is empty and all modes reduce to baseline behavior. Logging this outcome would attribute baseline noise to the selected mode.
  \item \emph{Embedding failures are excluded.} If the embedding API call for $\tau$ fails, the record cannot enter KNN retrieval and is dropped.
  \item \emph{Infrastructure failures are excluded.} Records from abnormal terminations, such as environment exceptions, are skipped because their outcomes cannot be attributed to the selected mode.
\end{enumerate}
These rules exclude attempts with empty initial memory, missing embeddings, or infrastructure failures, but do not remove task-selection effects or other confounding from the remaining observations.

\paragraph{Read protocol.}
On startup, the JSONL file is streamed to reload $\mathcal{L}_k$ in file order, using the following conventions for timestamp assignment and corrupted records.

\noindent\textit{Relative timestamps.}
Logical timestamps $t_i$ follow log order without being stored in the file, and each appended entry receives a timestamp $t_{\mathrm{new}}$ later than those of existing entries. Reloading assigns consecutive integer ranks that preserve order while ignoring elapsed time and previous timestamp spacing. For a fixed loaded log, anchoring decay at the latest neighborhood timestamp $t_{\max}$ rather than the current update index rescales every weight by the same positive factor, leaving normalized reward estimates unchanged.

\noindent\textit{Tolerance to corrupted lines.}
Malformed lines, such as those caused by interrupted writes, are skipped with a logged warning while loading continues, preventing a single corrupted record from invalidating the log.

\section{Worked Example: A Walkthrough of On-Demand Synthesis}
\label{app:walkthrough}

This appendix presents a constructed ALFWorld example at task update $k=177$, after the three-iteration warmup, illustrating the supported local-selection branch in which the neighborhood passes the similarity threshold, every mode has positive weight, and no exploration override occurs. The records, weighted outcomes, and synthesized guidance illustrate the procedure rather than document a measured execution trace.

\begingroup
\fontsize{9}{12}\selectfont
\setlength{\parskip}{2pt}
\begin{walkcard}[breakable=false]{Incoming task\hfill ALFWorld, $k=177$}
\textbf{Incoming task} ($\tau$, illustrative identifier \texttt{example\_heat\_egg}): \emph{``Your task is to: heat an egg and put it on the diningtable.''} Initial state lists \texttt{countertop 1, microwave 1, diningtable 1, cabinet 1--3, drawer 1--2, fridge 1, \ldots}
\end{walkcard}

\walkstep{1}{Candidate records from $M_k$ ($B=5$)}
Five candidate records illustrate the available experience (\cref{app:method:similarity}), including three distinct failures on related ``heat-and-place'' tasks. This pool provides context for the example; mode-specific retrieval occurs after mode selection in Step 3.

\begin{walkcard}{Retrieved execution records}
\begin{description}[leftmargin=30pt,labelwidth=23pt,labelsep=7pt,itemsep=3pt,parsep=0pt,topsep=0pt,font=\normalfont\sffamily\bfseries\color{papernavy}]
\walkrecord{1}{(sim $0.92$, $y=0$, update $175$): ``heat a tomato and put it on the diningtable'' --- issued \texttt{heat tomato 1 with microwave 1} without first taking the tomato; the action had no effect because the object was not held.}

\walkrecord{2}{(sim $0.88$, $y=0$, update $176$): ``heat a potato and put it on the diningtable'' --- issued \texttt{move potato 1 to diningtable 1} before any heating step.}

\walkrecord{3}{(sim $0.86$, $y=0$, update $174$): ``heat a mug and put it on the countertop'' --- searched cabinets and drawers for the microwave instead of going directly; exceeded step budget.}

\walkrecord{4}{(sim $0.82$, $y=1$, update $173$): ``heat an apple and put it on the diningtable'' (contrasting success).}

\walkrecord{5}{(sim $0.75$, $y=1$, update $172$): ``heat a bread and put it on the diningtable.''}
\end{description}
\end{walkcard}

\walkstep{2}{Local outcome estimates ($\hat B=10$)}
This illustration uses recency-weighted similarity weights $w_i = \max(\mathrm{sim}(\tau_i, \tau), 0)\cdot \gamma^{t_{\max}-t_i}$ with $\gamma=0.9$ and cost weight $\lambda=0.20$. Experimental defaults are listed separately in \cref{app:method:hyperparams}. In this constructed neighborhood, direct-use records have $\kappa=0$ and every synthesis-mode record has one realized synthesis call ($\kappa=1$). The per-mode estimates are shown with the UCB scores below.

\Needspace{13\baselineskip}
\walkstep{3}{UCB mode selection}
The controller adds the exploration bonus $c\sqrt{\log|\mathcal{N}_{\hat{B}}|/n_\mu}$ with $c=0.5$. Bonuses are rounded to two decimals.

\begin{walkcard}[breakable=false]{Local estimates and selection scores}
\centering
\setlength{\tabcolsep}{5pt}
\renewcommand{\arraystretch}{1.1}
\begin{tabular}{@{}lcccc@{}}
\toprule
Mode & $n_\mu$ & \shortstack{Weighted\\mean $y$} & Cost-adjusted $\hat r$ & UCB score \\
\midrule
none & $2$ & $0.10$ & $0.10$ & $0.10+0.54=0.64$ \\
imitation & $4$ & $0.30$ & $0.30-0.20=0.10$ & $0.10+0.38=0.48$ \\
strategic & $2$ & $0.50$ & $0.50-0.20=0.30$ & $0.30+0.54=0.84$ \\
\rowcolor{metabg!7}
\textbf{corrective} & $2$ & $0.85$ & $0.85-0.20=0.65$ & $0.65+0.54=\mathbf{1.19}$ \\
\bottomrule
\end{tabular}
\end{walkcard}

The controller selects $\mu^\star = \mathrm{corrective}$ because it has the highest score, drawing on candidate records that provide three distinct failure modes for analysis. By comparison, $\mathrm{imitation}$ abstracts behavior from successful records without explicitly analyzing those failures.

\Needspace{10\baselineskip}
\walkstep{4}{On-demand abstraction $\sigma_\tau=\phi(H_\tau,\mu^\star)$}
The selected $\mathrm{corrective}$ mode uses the top-3 failures ($j=1,2,3$) and contrasting success ($j=4$). The synthesis operator $\phi$ then makes one LLM call ($\kappa(\mathrm{corrective})=1$) to extract avoidance rules:

\begin{walkcard}{Avoidance Rules (corrective reflection)}
\begin{samepage}
\textcolor{papernavy}{\textbf{Rule 1.}} \texttt{Do NOT heat <object> unless you are holding it.}

From $j=1$: the heating action failed because the object was not held. Corrective action: \texttt{take} the object and keep it in inventory until \texttt{heat}.
\end{samepage}

\walkdivider
\begin{samepage}
\textcolor{papernavy}{\textbf{Rule 2.}} \texttt{Do NOT place <object> on diningtable before heating.}

From $j=2$: placement alone leaves the heating requirement unmet; completing the goal would require additional actions. Order of operations: heat first, then transport, then place.
\end{samepage}

\walkdivider
\begin{samepage}
\textcolor{papernavy}{\textbf{Rule 3.}} \texttt{Do NOT search cabinets/drawers for microwave 1.}

From $j=3$: microwave 1 is listed in the initial observation; use \texttt{go to microwave 1} directly.
\end{samepage}

\walkdivider
\textbf{Contrast ($j=4$, success).} \texttt{take} $\to$ \texttt{go to microwave 1} $\to$ \texttt{heat with microwave 1} $\to$ \texttt{go to diningtable 1} $\to$ \texttt{move to diningtable 1}.
\end{walkcard}

\Needspace{15\baselineskip}
\walkstep{5}{Illustrative execution under $\sigma_\tau$}
Under ALFWorld's action semantics, \texttt{heat} operates on a held object. The following constructed trajectory $e_\tau^{(177)}$ applies all three rules; bracketed observations are abbreviated illustrations:

\begin{walkcard}[breakable=false]{Execution trace\hfill Illustrative, 6 steps}
\setlength{\tabcolsep}{0pt}
\renewcommand{\arraystretch}{1.1}
\begin{tabularx}{\linewidth}{@{}r@{\hspace{10pt}}p{0.56\linewidth}Y@{}}
\textbf{Step} & \textbf{Action} & \textbf{Observation / guidance} \\
\midrule
1 & \texttt{> go to countertop 1} & [egg 1, knife 1] \\
2 & \texttt{> take egg 1 from countertop 1} & Rule 1: hold before heating \\
3 & \texttt{> go to microwave 1} & Rule 3: go directly \\
4 & \texttt{> heat egg 1 with microwave 1} & [heated] \\
5 & \texttt{> go to diningtable 1} & Rule 2: place after heating \\
6 & \texttt{> move egg 1 to diningtable 1} & [\textbf{success}] \\
\end{tabularx}
\end{walkcard}

\textbf{Illustrative outcome:} $y_\tau^{(177)}=1$ in 6 steps.

\Needspace{13\baselineskip}
\walkstep{6}{Memory update and discard}
The trajectory record enters $M_k$, the mode and outcome enter $\mathcal L_k$, and the synthesized abstraction is discarded:

\begin{walkcard}[breakable=false]{Persistent records after execution}
\textbf{Trajectory memory}
\[
M_{178}\gets M_{177}\cup\{(\tau,e_\tau^{(177)},1)\}.
\]
\textbf{Outcome log}
\[
\mathcal L_{178}\gets\mathcal L_{177}\cup
\{(\tau,\mathrm{corrective},1,0.80,t_{\mathrm{new}})\},
\]
with $r=1-0.20$. The temporary guidance $\sigma_\tau$ is discarded and $S_k$ remains empty, preventing direct reuse of this abstraction as a stored memory object.
\end{walkcard}

\textbf{Reading the example.} Local outcome estimates and cost-adjusted UCB scores determine the selected mode, after which corrective synthesis analyzes three failure patterns in the constructed records: heating without holding the object, premature placement, and inefficient search. A contrasting success helps formulate avoidance rules. The alternative modes emphasize reusable behavior from successful executions, higher-level planning, or direct factual context without synthesis. Although the temporary abstraction $\sigma_\tau$ is stored in neither $\mathcal{L}_k$ nor $M_k$, its effects remain in the recorded trajectory and may inform later synthesis, while the outcome log influences subsequent mode selection.
\color{black}
\endgroup

\Needspace{8\baselineskip}
\section{Experimental Setup Details}
\label{app:setup-details}

This appendix details the task splits, training and evaluation protocol, backbone configurations, and AutoRefine implementation used in the experiments. All memory and controller updates use training experience; held-out outcomes remain separate from the update process.

\subsection{Per-Benchmark Task Counts and Partitioning}
\label{app:setup-details:tasks}

For each benchmark, a fixed \emph{training set} supplies interaction experience for memory and controller-state updates with backbone parameters held fixed, while a separate fixed \emph{held-out set} is reserved for evaluation. Where available, we adopt community-standard splits from prior work~\citep{shinn2023reflexion, zhao2024expel, chen2024automanual, fang2025memp}, with task counts reported in \cref{tab:task-counts} and the per-iteration sampling procedure described below.

\begin{table}[h]
\centering
\caption{Per-benchmark training and held-out task counts used in this study.}
\label{tab:task-counts}
\small
\renewcommand{\arraystretch}{1.15}
\begin{tabularx}{\linewidth}{lccY}
\toprule
Benchmark & Training tasks & Held-out tasks & Evaluation set \\
\midrule
ALFWorld     & $180$ & $134$ & $134$ unseen tasks~\citep{shinn2023reflexion, zhao2024expel} \\
WebShop      & $200$ & $100$ & $100$ held-out instructions \\
ScienceWorld & $150$ & $60$  & $2$ held-out instances per task type \\
\bottomrule
\end{tabularx}
\end{table}

\paragraph{ALFWorld.}
ALFWorld~\citep{shridhar2020alfworld} provides $3{,}553$ training games and $134$ unseen evaluation games across six categories: \textit{Pick \& Place}, \textit{Examine in Light}, \textit{Clean \& Place}, \textit{Heat \& Place}, \textit{Cool \& Place}, and \textit{Pick Two \& Place}. We adopt the $134$-task held-out split used in prior work~\citep{shinn2023reflexion, zhao2024expel} and draw $30$ tasks per category to obtain a balanced training set of $180$ tasks. Although AutoManual~\citep{chen2024automanual} uses smaller per-category splits, our setup follows the larger Reflexion/ExpeL evaluation protocol while keeping the training set manageable.

\paragraph{WebShop.}
WebShop~\citep{yao2022webshop} provides $12{,}087$ crowd-sourced shopping instructions paired with target product descriptions and structured attributes. We use a fixed $100$-instruction held-out set. This evaluation size differs from the original benchmark's $500$-instruction test split~\citep{yao2022webshop}. We uniformly sample $200$ training instructions from the remaining $11{,}987$ using a fixed random seed to ensure a common training stream across methods.

\paragraph{ScienceWorld.}
ScienceWorld~\citep{wang2022scienceworld} provides $30$ task types in elementary science, including thermodynamics, electricity, and biology, with multiple parameterized variations of each type. We construct a stratified partition by randomly selecting $N_{\mathrm{train}} = 5$ training instances and $N_{\mathrm{held}} = 2$ nonoverlapping held-out instances for each of the $30$ types. The resulting $150$ training tasks and $60$ held-out tasks cover every task type, enabling evaluation on unseen variations within each type. Sampling once with a fixed random seed gives all methods the same partition, supporting controlled comparisons within this paper despite the difference from prior protocols.

\paragraph{Per-iteration training subset.}
At each experimental iteration $\ell$, we sample $N_{\mathrm{train/iter}} = 50$ training tasks uniformly without replacement and attempt them sequentially in an order determined by a fixed random seed. These episodes alone may update each method's memory and related state according to its update procedure, yielding $500$ training episodes per benchmark across $K = 10$ iterations. All methods receive the same sample at iteration $\ell$, controlling for training-task composition in comparisons of memory designs.

After each iteration's training updates, each method is evaluated on the full held-out set with memory and related state frozen. Held-out episodes use a separate log that is never carried forward and cannot update trajectory or abstraction memory, $\mathcal{L}_k$, retrieval metadata, maintenance scores, the skill evolver, or pruning. Each episode permits up to $50$ environment steps, with attempts exceeding this budget terminated as failures ($y = 0$).

\subsection{Backbone Configurations}
\label{app:setup-details:backbones}

We evaluate three open-weight backbones from distinct model families---Llama-4-Scout, Qwen3-Coder-Next, and Gemma-4-31B-it---running locally through OpenAI-compatible inference endpoints with the hardware and context lengths in \cref{tab:backbones}. All methods share the same agent prompt and per-step output budget to control these factors when comparing memory configurations.

\begin{table}[h]
\centering
\caption{Backbone configurations.}
\label{tab:backbones}
\small
\renewcommand{\arraystretch}{1.15}
\begin{tabular}{lll}
\toprule
Model & Inference hardware & Context length \\
\midrule
Llama-4-Scout~\citep{meta2025llama4}              & $2 \times$ NVIDIA H100 & 64K  \\
Qwen3-Coder-Next~\citep{cao2026qwen3}             & $4 \times$ NVIDIA H100 & 256K \\
Gemma-4-31B-it~\citep{deepmind2026gemma4}         & $4 \times$ NVIDIA L40S & 256K \\
\bottomrule
\end{tabular}
\end{table}

\begin{sloppypar}
\paragraph{Embedding model.}
Qwen3-Embedding-8B~\citep{zhang2025qwen3embedding} produces task embeddings for trajectory retrieval (\cref{app:method:similarity}) and KNN lookup through an OpenAI-compatible embeddings endpoint on the same cluster. Anti-persistence makes one embedding call per task using the same encoder with every backbone, ensuring a common embedding representation for outcome-log similarity comparisons across configurations.
\end{sloppypar}

\subsection{AutoRefine Baseline Implementation}
\label{app:setup-details:autorefine}

Our AutoRefine baseline adapts the maintenance procedure of \citet{qiu2026autorefine}, using the same retrieval and per-iteration episode budgets as the other methods (\cref{subsec:setup}). The five components below include utility scoring, retrieval, pruning, synthesis, and the update loop. This baseline therefore differs from abstraction memory in more than maintenance alone. Its results are reported in \cref{tab:head-to-head} and \cref{fig:maintenance}.

\paragraph{(i) Per-skill utility score.}
For each stored skill $s$, the metadata triple $(n^{\mathrm{inj}}_s,\, n^{\mathrm{succ}}_s,\, t^{\mathrm{create}}_s)$ records its prompt-inclusion count, the number of successful episodes among those inclusions, and its creation iteration. We define skill utility as
\begin{equation}
\mathrm{score}(s) \;=\; \underbrace{\frac{n^{\mathrm{succ}}_s}{n^{\mathrm{inj}}_s + \varepsilon}}_{\text{effectiveness}} \;\cdot\; \underbrace{\log\!\big(1 + n^{\mathrm{inj}}_s\big)}_{\text{frequency}},
\end{equation}
with $\varepsilon = 10^{-3}$, so the frequency factor discounts skills supported by only one or two retrievals, while skills that have never appeared in a successful episode receive zero. This metadata is stored in a JSON file alongside the skill files.

\paragraph{(ii) Multi-query retrieval with similarity floor and MMR diversification.}
For each episode, the backbone rewrites the task description into three queries emphasizing \emph{strategy}, \emph{pitfall}, and \emph{actions}, which join the original description to form four query strings. Each query is embedded with Qwen3-Embedding-8B (\cref{app:method:similarity}), and a candidate skill's relevance is its maximum cosine similarity across the queries. After discarding candidates below the similarity threshold of $0.30$, we re-rank the remainder using Maximal Marginal Relevance with diversity weight $\lambda_{\mathrm{MMR}} = 0.7$, selecting at each step the skill that maximizes
\begin{equation}
\lambda_{\mathrm{MMR}}\,\mathrm{rel}(s) \;-\; (1 - \lambda_{\mathrm{MMR}})\,\max_{s' \in \mathcal{S}_{\mathrm{sel}}}\,\mathrm{sim}(s, s'),
\end{equation}
and concatenating the top $5$ entries into the prompt. Candidates are drawn from the general-skill set, the common-mistakes set, and the benchmark's task-type-specific skill set.

\paragraph{(iii) Periodic percentile pruning.}
Every two iterations (at $\ell = 2, 4, \ldots, K$), we prune the bottom $20\%$ of eligible skills by utility score, with eligibility restricted to \emph{$n^{\mathrm{inj}}_s \ge 3$} to protect skills with little retrieval history. A single transaction removes the selected entries from the in-memory dictionary, on-disk skill files, and metadata file while invalidating their embedding caches.

\paragraph{(iv) Contrastive skill evolver.}
After each iteration, a single contrastive prompt extracts new skills from at most four failure episodes and four sampled success episodes, using a benchmark-specific system preamble for ALFWorld, WebShop, or ScienceWorld. Each generated skill must include a \texttt{Why it works} section that contrasts the success and failure evidence. The synthesizer's output is decoded as a JSON array through a four-stage fallback parser comprising direct parsing, fenced-block stripping, extraction of the last bracketed array, and per-object scanning; skills with non-slug names are renamed \texttt{dyn-$N$}, using the smallest $N$ above the maximum existing index. This single-call evolver replaces the original AutoRefine release's two-call procedure for reflecting on failures and extracting experience from successes, placing both forms of evidence in the shared context required by the contrastive prompt.

\paragraph{(v) Per-iteration update loop.}
Each AutoRefine iteration $\ell$ runs the following loop:
\begin{enumerate}[leftmargin=2em,itemsep=2pt]
    \item Reset the training retrieval log for the current iteration.
    \item Run the training episodes for iteration $\ell$, invoking the multi-query retrieval procedure in step (ii) for each episode.
    \item Update metadata from the training retrieval log. For each skill included in an episode's prompt, set $n^{\mathrm{inj}}_s \mathrel{+}= 1$ and, if that episode succeeded, $n^{\mathrm{succ}}_s \mathrel{+}= 1$.
    \item Run the contrastive evolver in step (iv) using only training outcomes. Add the new skills to the store and set $t^{\mathrm{create}}_s = \ell$.
    \item If $\ell > 0$ and $\ell$ is even, run percentile pruning (step (iii)) and write a metadata checkpoint.
    \item Freeze the resulting state and evaluate on the full held-out set. Held-out episodes may retrieve skills, but a separate log records their retrievals and outcomes. These records do not update skill files, metadata, the evolver, or pruning.
\end{enumerate}
The training retrieval log identifies each injection by the tuple (task type, task description, initial observation), and Step (3) credits a successful retrieval only when the skill appeared in that episode's prompt. Successes achieved without retrieving the skill therefore do not increase its utility score.

\section{Extended Cross-Backbone Results}
\label{app:full-results}

The main text reports per-iteration performance on Llama-4-Scout and Qwen3-Coder-Next (\cref{fig:divergence-pair}) and aggregate per-task LLM-call cost (\cref{fig:cost-summary}). This appendix extends that evaluation with success-rate and trajectory-length curves for Gemma-4-31B-it (\cref{app:full-results:gemma}), a cost breakdown for every $(\text{backbone} \times \text{benchmark})$ setting (\cref{app:full-results:cost}), and an analysis of whether descendants with flagged parents are more likely to be flagged themselves (\cref{app:full-results:descendant}). The latter describes historical associations relevant to the proposed error-reproduction pathway.

\subsection{Per-Iteration Curves on Gemma-4-31B-it}
\label{app:full-results:gemma}

\Cref{fig:gemma-pair} reports held-out success rates (left) and mean trajectory lengths (right) across memory-update iterations for Gemma-4-31B-it on all three benchmarks. These results extend the comparison for Llama-4-Scout and Qwen3-Coder-Next in \cref{fig:divergence-pair}.

\textbf{Success rate.}
Among the four memory configurations, anti-persistence achieved the highest final mean success on ALFWorld ($97.3\%$) and WebShop ($53.3\%$). Joint memory performed best on ScienceWorld ($56.3\%$, versus $54.1\%$ for anti-persistence). These final values match \cref{tab:head-to-head}, while the curves show that the relative ordering varied across earlier rounds.

\textbf{Trajectory length.}
Joint memory produced the shortest trajectories on ALFWorld, while trajectory memory did so on WebShop, followed closely by anti-persistence. Trajectory memory also produced the shortest trajectories on ScienceWorld, where anti-persistence and joint memory had comparable lengths. Within each benchmark, the success-rate and trajectory-length panels cover the same iteration range for every method.
\begin{figure}[!htbp]
\centering
\captionsetup{skip=5pt}
\captionsetup[subfigure]{skip=5pt}
\begin{subfigure}[t]{0.49\linewidth}
  \centering
  \includegraphics[width=\linewidth]{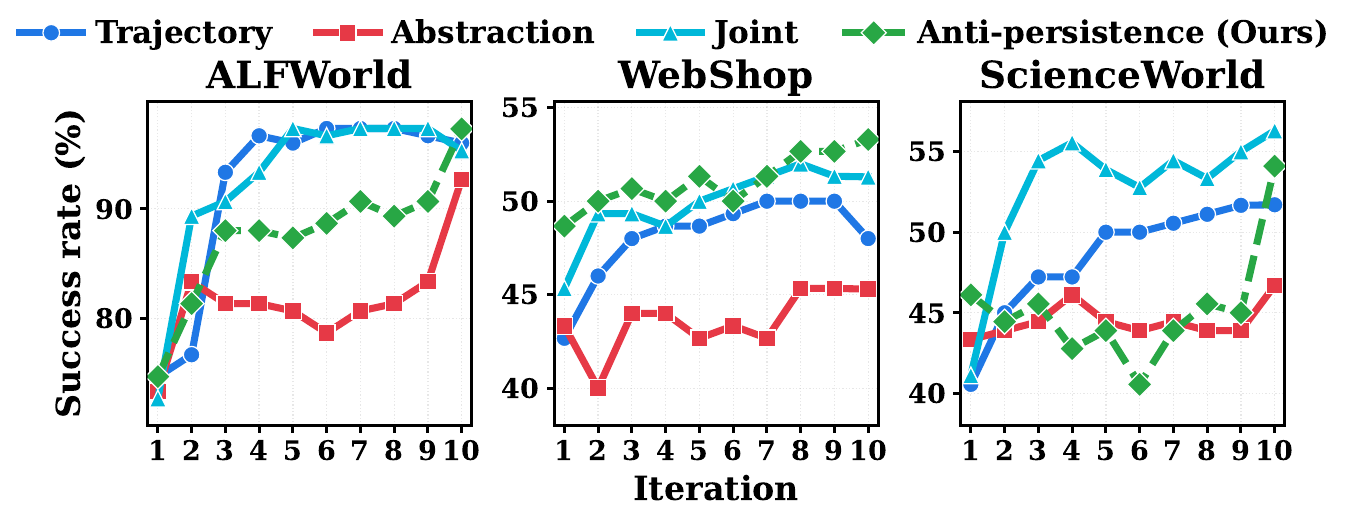}
  \caption{Per-iteration held-out success rate.}
  \label{fig:divergence-gemma}
\end{subfigure}
\hfill
\begin{subfigure}[t]{0.49\linewidth}
  \centering
  \includegraphics[width=\linewidth]{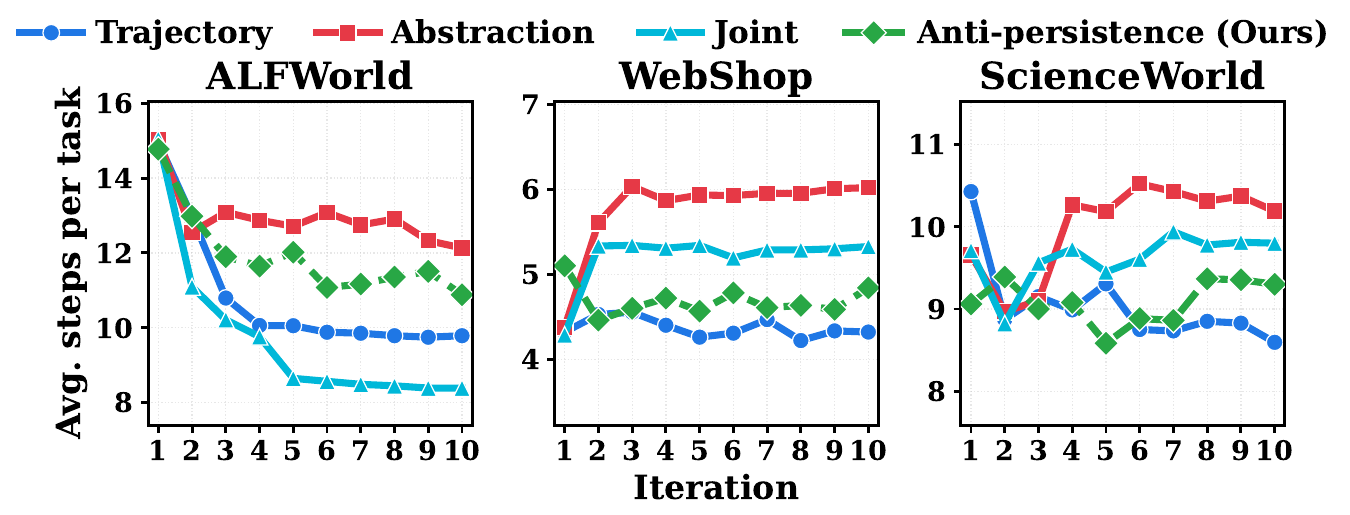}
  \caption{Per-iteration mean trajectory length.}
  \label{fig:avg-steps-gemma}
\end{subfigure}
\caption{Per-iteration performance on Gemma-4-31B-it across the three benchmarks.}
\label{fig:gemma-pair}
\end{figure}

\subsection{Per-Task LLM-Call Cost}
\label{app:full-results:cost}

\Cref{tab:cost-per-task} breaks down the aggregate cost in \cref{fig:cost-summary} by reporting average per-task LLM-call counts for all four memory configurations across the nine $(\text{backbone} \times \text{benchmark})$ settings. The metric includes all calls during task execution, including the additional synthesis calls under anti-persistence (\cref{app:method:modes}).

\begin{table}[t]
\centering
\caption{Average LLM-call cost per task after the final memory-update iteration (mean $\pm$ standard deviation over $3$ seeds). Lower is better.}
\label{tab:cost-per-task}
\fontsize{8.5}{10.5}\selectfont
\setlength{\tabcolsep}{3pt}
\renewcommand{\arraystretch}{1.2}
\begin{tabular*}{\linewidth}{@{\extracolsep{\fill}}lccccccccc@{}}
\toprule
 & \multicolumn{3}{c}{Llama-4-Scout} & \multicolumn{3}{c}{Qwen3-Coder-Next} & \multicolumn{3}{c}{Gemma-4-31B-it} \\
\cmidrule(lr){2-4} \cmidrule(lr){5-7} \cmidrule(lr){8-10}
Method & ALF & Web & Sci & ALF & Web & Sci & ALF & Web & Sci \\
\midrule
Traj. Mem.   & 17.87$\pm$0.40 & 4.21$\pm$0.09 & 5.25$\pm$0.10 & 10.90$\pm$0.15 & 6.12$\pm$0.03 & 11.41$\pm$0.32 & 11.87$\pm$0.19 & 4.85$\pm$0.06 & 13.00$\pm$0.05 \\
Abs. Mem.    & 16.74$\pm$4.77 & 4.53$\pm$0.27 & 1.33$\pm$1.15 & 13.64$\pm$0.03 & 6.91$\pm$0.10 & 15.08$\pm$1.05 & 13.46$\pm$1.04 & 6.04$\pm$0.17 & 13.80$\pm$0.59 \\
Joint Mem.   & 19.11$\pm$0.48 & 3.23$\pm$1.57 & 1.15$\pm$0.86 & 10.95$\pm$0.33 & 7.21$\pm$0.10 & 13.93$\pm$0.95 & 11.01$\pm$0.50 & 5.96$\pm$0.17 & 12.78$\pm$0.71 \\
\textbf{Anti-pers.} & 15.13$\pm$0.19 & 3.99$\pm$0.05 & 3.56$\pm$0.22 & 11.84$\pm$0.28 & 6.64$\pm$0.55 & 12.68$\pm$0.40 & 11.20$\pm$0.38 & 5.30$\pm$0.19 & 13.41$\pm$0.19 \\
\bottomrule
\end{tabular*}
\end{table}

\subsection{Provenance Associations and Operational Bias Flags}
\label{app:full-results:descendant}

We describe an audit of parent--descendant associations relevant to the pathway in \cref{app:dynamics}, combining historical exposure records with outcome-based flags assigned after evaluation. The equations below specify the intended quantities and their data requirements, while \cref{fig:descendant} retains the supplied aggregates, which have not been recomputed for this revision or verified against event-level provenance and labeling-eligibility logs. The audit therefore remains descriptive and cannot establish a causal reproduction effect or a long-run growth coefficient.

\paragraph{Descendant tracking.}
For a candidate $s=\phi(H)$, its historical parent set is
\[
P(s)=\bigcup_{m\in H}P(m),
\]
where $P(m)$ records the persistent abstractions present in the context when execution record $m$ was generated, retaining their historical identifiers even after the corresponding objects leave the current store $S_k$. Parent sets are recorded as analysis metadata when a candidate is created, without retroactively modifying the agent's memory. They document exposure and data lineage but do not establish that an abstraction caused a particular action or error.

\paragraph{Operational labeling.}
For each historically stored abstraction, the audit compares held-out episodes that retrieved it with episodes that did not. Let $n_s^+,n_s^-$ and $\bar y_s^+,\bar y_s^-$ denote the respective episode counts and mean success rates. Define the set of abstractions eligible for labeling as
\[
\widehat{\mathcal S}_{\mathrm{lab}}
=\{s:n_s^+\ge n_{\min},\ n_s^->0\}.
\]
For $s$ in this set, assign
\[
\widehat b(s)=\mathbf 1\{\bar y_s^- - \bar y_s^+\ge\delta\},
\qquad n_{\min}=3,\quad\delta=0.05,
\]
and let $\widehat{\mathcal B}=\{s\in\widehat{\mathcal S}_{\mathrm{lab}}:\widehat b(s)=1\}$. Abstractions outside the eligible set remain unclassified throughout the audit, and both comparison groups must exist to avoid assigning labels from undefined means. These unadjusted operational flags instantiate \cref{def:biased} without estimating causal harm, since retrieval may, for example, concentrate certain abstractions on more difficult tasks.

The audit uses held-out logs collected with memory and related state frozen, with labels and provenance analyses computed after the runs so that they cannot influence memory updates, maintenance scores, pruning, or subsequent behavior. Comparisons between the groups below require a common evaluation protocol and common audit thresholds.

\paragraph{Empirical conditional rates.}
Let $\mathcal V$ contain historically stored descendants that satisfy the labeling-eligibility conditions above and have a recorded historical parent set. Define $F(s)=\mathbf 1\{P(s)\cap\widehat{\mathcal B}\ne\emptyset\}$. The intended descriptive group rates are
\[
\widehat\alpha_{\mathrm{emp}}
=\frac{\sum_{s\in\mathcal V}F(s)\widehat b(s)}
{\sum_{s\in\mathcal V}F(s)},
\qquad
\widehat\alpha_0
=\frac{\sum_{s\in\mathcal V}(1-F(s))\widehat b(s)}
{\sum_{s\in\mathcal V}(1-F(s))},
\]
provided the corresponding denominator is positive, with a zero denominator leaving the rate undefined. The first group contains descendants with at least one flagged parent, whereas the second contains those with none; because the latter may include unclassified parents or an empty parent set, it does not constitute a group with verified benign ancestry.

These pooled observational rates describe conditional associations across the logged states rather than estimating $\alpha$ for a single fixed post-maintenance state. They also condition on prior storage and labeling eligibility, so they cannot separately identify candidate-level generation and admission rates. Comparing against $\widehat\alpha_0$ leaves task difficulty, retrieval selection, parent-set size, and other potential confounders uncontrolled, while shared ancestors and evaluation episodes introduce dependence among observations.

\paragraph{Results and scope.}
Descendants with at least one flagged parent had a higher flag rate in most of the six logged settings (\cref{fig:descendant}), with the largest reported gap, $\widehat\alpha_{\mathrm{emp}}-\widehat\alpha_0=+0.55$, occurring on Llama-4-Scout with ALFWorld. This association is consistent with the proposed pathway but identifies neither inherited error content nor a causal effect of parent exposure. The absence of group counts, labeling coverage, and an uncertainty-estimation procedure limits statistical interpretation of the aggregates, which cannot establish a causal amplification factor, a retention probability, or a threshold for long-run instability.

\begin{figure}[!htbp]
\centering
\captionsetup{skip=2pt}
\captionsetup[subfigure]{skip=2pt}
\begin{minipage}[c]{0.46\linewidth}
  \centering
  \includegraphics[width=\linewidth]{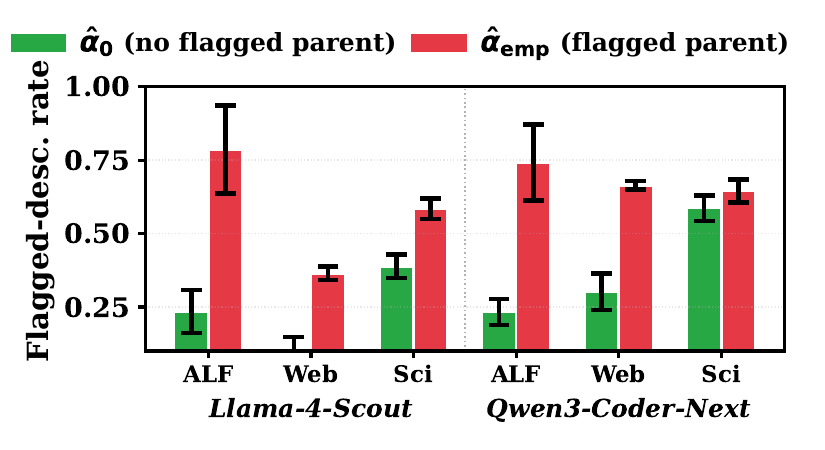}
  \caption{\small Observational flagged-descendant rates across six settings. Groups have at least one flagged parent or no flagged parent; absence of a flag does not establish benign ancestry. The largest reported rate gap is $+0.55$. Error bars are retained as reported, but their estimation procedure is unavailable with the aggregate data.}
  \label{fig:descendant}
\end{minipage}\hfill
\nextfloat
\begin{minipage}[c]{0.52\linewidth}
  \centering
  \begin{subfigure}[t]{0.48\linewidth}
    \centering
    \includegraphics[width=\linewidth]{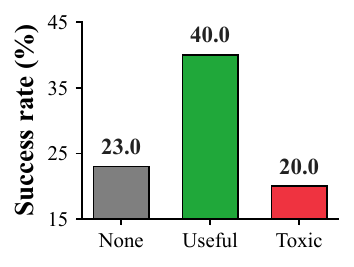}
    \caption{\small AppWorld ($n=30$).}
    \label{fig:additional-controlled-a}
  \end{subfigure}\hfill
  \begin{subfigure}[t]{0.48\linewidth}
    \centering
    \includegraphics[width=\linewidth]{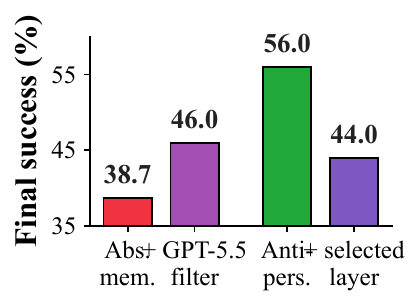}
    \caption{\small WebShop ($K=10$).}
    \label{fig:additional-controlled-b}
  \end{subfigure}
  \caption{\small Controlled abstraction-content and retention analyses; y-axes are truncated.}
  \label{fig:additional-controlled}
\end{minipage}
\end{figure}
\vspace{-6pt}

\section{Additional Controlled and Statistical Analyses}
\label{app:additional-analyses}

Three additional analyses examine the immediate effect of supplied skill content (\cref{fig:additional-controlled-a}), selective retention of filtered skills (\cref{fig:additional-controlled-b}), and consistency of performance differences across the nine main evaluation settings (\cref{tab:setting-level-stats}). The two controlled analyses report aggregate success rates, with task-level and seed-level observations unavailable for estimating uncertainty.

\subsection{Effects of Skill Content on AppWorld}
\label{app:additional-analyses:intervention}

We tested the immediate effect of skill content on 30 randomly selected AppWorld tasks~\citep{trivedi2024appworld}, using matched pairs of reusable skills constructed by GPT-5.5 before evaluation. Each pair shared the same intended scope but contained a useful version with a correct procedural rule and a toxic version with an incorrect rule; these labels denote constructed conditions rather than the outcome-based flags in \cref{app:full-results:descendant}. With the skill libraries frozen, Qwen3-Coder-Next attempted the same tasks under three conditions: no supplied skills, useful skills, and toxic skills.

With the acting model, task set, and intended skill scope held fixed, aggregate success was $23.0\%$ without supplied skills, $40.0\%$ with useful skills, and $20.0\%$ with toxic skills (\cref{fig:additional-controlled-a}). These results characterize the immediate effects of supplied abstraction content on performance without establishing whether the resulting trajectories generate biased descendants or sustain error reproduction across iterations.

\subsection{Selective Skill Retention on WebShop}
\label{app:additional-analyses:retention}

We examined selective skill retention on WebShop with Qwen3-Coder-Next under two configurations: GPT-5.5 assessed each newly generated skill before admission to persistent abstraction memory in the first, while the selected skills were added to anti-persistence as a frozen persistent-memory layer in the second.

Within this auxiliary evaluation, filtering increased the final success rate of abstraction memory from $38.7\%$ to $46.0\%$, whereas adding the frozen layer to anti-persistence reduced its final success rate from $56.0\%$ to $44.0\%$ (\cref{fig:additional-controlled-b}).

The tested filter and retention protocol thus improved the abstraction-memory baseline but reduced performance when the filtered skills were added to anti-persistence. Because the added layer remained frozen, its performance effect does not establish descendant generation or error reproduction across iterations.

\subsection{Setting-Level Consistency Across the Main Evaluation Grid}
\label{app:additional-analyses:stats}

To assess consistency across the nine fixed $(\text{backbone}\times\text{benchmark})$ settings in \cref{tab:head-to-head}, we computed the difference between anti-persistence and each baseline within every setting. We then summarized these paired differences by their equally weighted mean in percentage points and the number of settings with a positive difference.

For an exploratory statistical summary, we computed nominal 95\% $t$-intervals, two-sided paired $t$-tests, and exact two-sided Wilcoxon signed-rank tests, applying Holm correction separately to the seven baseline comparisons within each test family while leaving the intervals unadjusted. All calculations used the one-decimal cell means in \cref{tab:head-to-head}, with NumPy 2.5.0 and SciPy 1.18.1.

\begin{table}[t]
\centering
\caption{\small Exploratory comparisons across the nine fixed settings in \cref{tab:head-to-head}. $\Delta$ denotes anti-persistence minus the named baseline, measured in percentage points. CIs are nominal, unadjusted 95\% $t$-intervals, and all tests are two-sided. Holm correction covers the seven baseline comparisons within each test family. Shared backbones and benchmarks limit the inferential interpretation of these statistics.}
\label{tab:setting-level-stats}
\small
\setlength{\tabcolsep}{3.5pt}
\renewcommand{\arraystretch}{1.3}
\begin{tabular*}{\linewidth}{@{\extracolsep{\fill}}lccccccc@{}}
\toprule
Baseline & Mean $\Delta$ & 95\% CI & \multicolumn{2}{c}{Paired $t$-test} & \multicolumn{2}{c}{Wilcoxon signed-rank} & Positive \\
\cmidrule(lr){4-5}\cmidrule(lr){6-7}
 & (pp) & & Raw $p$ & Holm $p$ & Raw $p$ & Holm $p$ & cells \\
\midrule
ReAct               & $+19.2$ & $[14.1,24.4]$ & $<0.0001$ & $0.0002$ & $0.0039$ & $0.0273$ & $9/9$ \\
Reflexion           & $+7.0$ & $[2.7,11.4]$  & $0.0060$ & $0.0179$ & $0.0078$ & $0.0273$ & $8/9$ \\
ExpeL               & $+14.9$ & $[3.6,26.1]$  & $0.0157$ & $0.0179$ & $0.0039$ & $0.0273$ & $9/9$ \\
AutoRefine          & $+13.2$ & $[8.7,17.6]$  & $0.0001$ & $0.0007$ & $0.0039$ & $0.0273$ & $9/9$ \\
Trajectory memory   & $+6.4$ & $[2.8,10.0]$  & $0.0036$ & $0.0145$ & $0.0039$ & $0.0273$ & $9/9$ \\
Abstraction memory  & $+10.9$ & $[7.6,14.2]$  & $<0.0001$ & $0.0004$ & $0.0039$ & $0.0273$ & $9/9$ \\
Joint memory        & $+10.6$ & $[3.6,17.5]$  & $0.0079$ & $0.0179$ & $0.0195$ & $0.0273$ & $8/9$ \\
\bottomrule
\end{tabular*}
\end{table}

Anti-persistence achieved higher success in all nine settings against five baselines and in eight settings against Reflexion and joint memory (\cref{tab:setting-level-stats}), with mean differences ranging from $+6.4$ to $+19.2$ percentage points. The direction of its advantage was therefore consistent across most or all settings for each baseline.

Although all seven Holm-adjusted $p$-values were below $0.05$ in each test family, the shared backbones and benchmarks limit the treatment of the nine settings as independent observations. Holm correction addresses multiple comparisons without resolving this dependence, so the nominal intervals and $p$-values remain exploratory summaries rather than evidence for population-level significance. They neither quantify seed-level uncertainty within settings nor justify extrapolation to untested models and benchmarks.

\section{Sensitivity and Robustness Studies}
\label{app:sensitivity}

We examine sensitivity to the two retrieval budgets introduced in \cref{app:method:hyperparams}: the trajectory budget $B$, which controls the records supplied for direct reuse or synthesis (\cref{app:sensitivity:B}), and the outcome-log budget $\hat{B}$, which controls the records used for mode selection (\cref{app:sensitivity:Bhat}).

\subsection{Trajectory Retrieval Budget $B$ Sensitivity}
\label{app:sensitivity:B}

The base trajectory retrieval budget $B$ limits the records retrieved under the mode-specific rules (\cref{app:method:similarity}): up to $B$ records for direct reuse under $\mathrm{none}$, up to $B$ successes for imitation and strategic synthesis, or up to $B$ failures and one optional success for corrective synthesis. To assess the effect of supplying fewer or more records as context, we evaluate $B \in \{1, 5, 10\}$ on Qwen3-Coder-Next with ALFWorld and WebShop, using three seeds per setting.

\textbf{Observations.}
\Cref{fig:sensitivity-B} shows that both tested alternatives to the default retrieval budget yielded lower success rates.

\noindent\textit{(i) The smallest retrieval budget underperformed.}
At $B = 1$, direct reuse receives at most one record; synthesis receives at most one success for imitation or strategic mode, or one failure and an optional success for corrective mode. The resulting success rates of $86.0\%$ on ALFWorld and $46.0\%$ on WebShop were below the default, consistent with greater sensitivity to individual examples when only a few records are available.

\noindent\textit{(ii) The largest retrieval budget also underperformed.}
At $B = 10$, success rates were $91.3\%$ on ALFWorld and $48.7\%$ on WebShop, again below the default. Adding more weakly related trajectories may dilute the useful context for direct reuse or synthesis, although this sweep does not isolate that explanation.

\noindent\textit{(iii) The default setting $B = 5$ performed best among the tested values.}
The default budget gave the highest observed success rates on both benchmarks, favoring an intermediate tested value as in the recency-decay analysis (\cref{fig:sensitivity-gamma}). This pattern is consistent with a trade-off between local relevance and coverage of retrieved experience.

\begin{figure}[t]
\centering
\captionsetup{skip=3pt}
\captionsetup[subfigure]{skip=3pt}
\begin{subfigure}{0.49\linewidth}
  \centering
  \includegraphics[width=\linewidth]{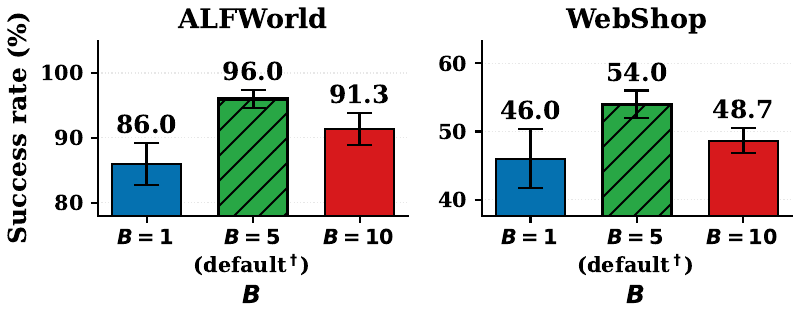}
  \caption{Trajectory retrieval budget $B$. The default $B = 5$ yields the highest success rate among the tested values.}
  \label{fig:sensitivity-B}
\end{subfigure}
\hfill
\begin{subfigure}{0.49\linewidth}
  \centering
  \includegraphics[width=\linewidth]{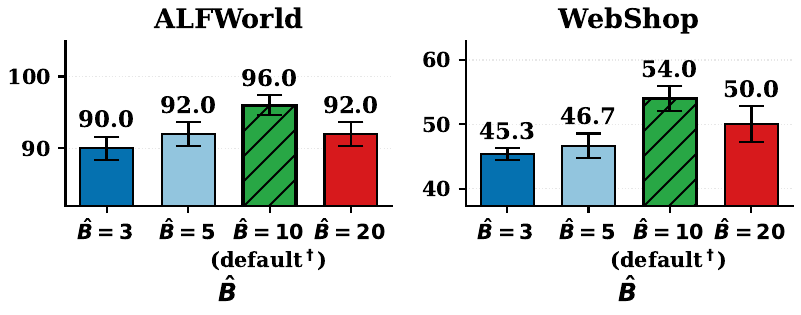}
  \caption{Outcome-log retrieval budget $\hat{B}$. The default $\hat{B} = 10$ yields the highest success rate among the tested values.}
  \label{fig:sensitivity-Bhat}
\end{subfigure}
\caption{Sensitivity to trajectory retrieval budget $B$ and outcome-log retrieval budget $\hat{B}$ on Qwen3-Coder-Next. Bars show means over $3$ seeds; error bars show standard deviations across seeds.}
\label{fig:sensitivity-budgets}
\end{figure}

\subsection{Outcome-Log Retrieval Budget $\hat{B}$ Sensitivity}
\label{app:sensitivity:Bhat}
The outcome-log retrieval budget $\hat{B}$ limits the size of the bandit's local neighborhood and thus the number of nearby observations available for mode selection (\cref{app:method:hyperparams}). We evaluate $\hat{B} \in \{3, 5, 20\}$ on Qwen3-Coder-Next with ALFWorld and WebShop, using three seeds per setting, and compare these alternatives with the default $\hat{B} = 10$ configuration in \cref{tab:head-to-head}.

\textbf{Observations.}
\Cref{fig:sensitivity-Bhat} shows the same qualitative pattern as the trajectory-budget sweep, with the default neighborhood size outperforming both smaller and larger tested values.

\noindent\textit{(i) Smaller neighborhoods yielded lower success rates.}
With only a few observations available for local reward estimation at $\hat{B}=3$, held-out success was $90.0\%$ on ALFWorld and $45.3\%$ on WebShop. Increasing the neighborhood to $\hat{B}=5$ improved these rates to $92.0\%$ and $46.7\%$, respectively, although both remained below the default, potentially reflecting greater sensitivity to individual outcomes in small neighborhoods.

\noindent\textit{(ii) The largest neighborhood also underperformed.}
At $\hat{B}=20$, success rates were $92.0\%$ on ALFWorld and $50.0\%$ on WebShop, again below the default. Including more weakly related tasks may reduce the relevance of local reward estimates, although the sweep does not directly test this explanation.

\noindent\textit{(iii) The default setting $\hat{B}=10$ performed best among the tested values.}
The default neighborhood size gave the highest observed success rates on both benchmarks, complementing the trajectory-budget sweep in \cref{app:sensitivity:B} and remaining consistent with a trade-off between local relevance and evidence coverage. These results support the default budgets within the tested ranges and settings without establishing a general optimum.

\section{Limitations}
\label{app:limitations}

Anti-persistence prevents direct reuse of stored abstractions, but temporary guidance can still influence later synthesis through trajectories. Its performance also depends on retrieval quality and experience coverage. Our main iterative experiments cover three text-based benchmarks with fixed backbones and a fixed improvement procedure; broader generalization and longer-term gains remain untested. The two controlled auxiliary analyses lack task-level and seed-level uncertainty estimates, while the setting-level statistics remain exploratory (\cref{app:additional-analyses}). Discarding abstractions also forgoes the benefits of reusing reliable knowledge across tasks. Future work could examine when external verification or uncertainty estimates justify selective persistence.
\end{document}